\documentclass[letterpaper]{article}
\usepackage{ijcai13}
\usepackage{times}
\usepackage{helvet}
\usepackage{courier}
\usepackage{hyperref}
\usepackage{graphicx}
\usepackage{latexsym}
\usepackage{amsmath,amssymb,amsfonts,amsthm}
\usepackage{graphicx}
\usepackage{hyperref}
\usepackage{multirow}
\newtheorem{theorem}{Theorem}
\newtheorem{lemma}{Lemma}
\newtheorem{proposition}{Proposition}
\newtheorem{example}{Example}
\newtheorem{corollary}{Corollary}

\newcommand{\myOmit}[1]{}

\newcommand{\Alt}{\mbox{\sc AltPolicy}}

\newcommand{\BestPref}{\mbox{\sc BestPref}}
\newcommand{\Random}{\mbox{\sc Random}}

\newcommand{\reals}{\mathbb{R}}

\newcommand{\vect}[1]{\mbox{\boldmath$#1$}}

\DeclareMathOperator{\sw}{sw}
\DeclareMathOperator{\dsw}{dsw}

\newcommand{\prob}[1]{\textbf{P}\left(#1\right)}
\newcommand{\expect}[1]{\textbf{E}\left(#1\right)}

{\makeatletter
 \gdef\xxxmark{%
   \expandafter\ifx\csname @mpargs\endcsname\relax % in minipage?
     \expandafter\ifx\csname @captype\endcsname\relax % in figure/caption?
       \marginpar{xxx}% not in a caption or minipage, can use marginpar
     \else
       xxx % notice trailing space
     \fi
   \else
     xxx % notice trailing space
   \fi}
 \gdef\xxx{\@ifnextchar[\xxx@lab\xxx@nolab}
 \long\gdef\xxx@lab[#1]#2{{\bf [\xxxmark #2 ---{\sc #1}]}}
 \long\gdef\xxx@nolab#1{{\bf [\xxxmark #1]}}
}

\begin{document}
\title{A Social Welfare Optimal Sequential Allocation Procedure}
\author{Thomas Kalinowski \\
Universit\"{a}t Rostock\\
Rostock, Germany\\
thomas.kalinowski@uni-rostock.de
\And 
Nina Narodytska
\and
Toby Walsh \\
NICTA and UNSW \\
Sydney, Australia \\
\{nina.narodytska,toby.walsh\}@nicta.com.au
}
\maketitle

\begin{abstract}
We consider a simple sequential allocation
procedure for sharing indivisible items between
agents in which agents take turns to pick items.
Supposing additive utilities
and independence between the agents, we
show that the expected utility of each agent
is computable in polynomial time. Using
this result, we prove that the expected utilitarian social welfare
is maximized when agents take alternate turns. We also
argue that this mechanism remains optimal when
agents behave strategically.
\end{abstract}

\section{Introduction}

There exist a variety of mechanisms to
share indivisible goods between agents
without side payments 
\cite{bfscw2000,hpscw2002,peftd2003,bkjtp2004,undercut}. 
One of the simplest is simply to let the agents take turns
to pick an item.
%%in picking the item that they value most.
%Such a mechanism has been used to allocate courses
%to students at Harvard \cite{Budish12:Multi}. 
This mechanism is parameterized by a policy,
the order in which agents take turns.
In the alternating policy, agents take turns in
a fixed order, whilst in the balanced alternating
policy, the order of agents reverses each
round. Bouveret and Lang (\citeyear{bouveretgeneral})
study a simple model of this mechanism
in which agents have strict preferences
over items, and utilities are additive.
They conjecture that  computing
the expected social welfare for a given policy
is NP-hard supposing all preference orderings are
equally likely.
Based on %computer
simulation for up to 12 items,
they conjecture
that the alternating policy maximizes the expected utilitarian
social welfare for Borda utilities, and prove it does so
asymptotically. We close
both conjectures. Surprisingly, we prove that the expected
utility of each agent can be computed in polynomial time
for any policy and utility function.
Using this result, we prove that the alternating policy maximizes
the expected utilitarian social welfare
for any number of
items and any
linear utility function including Borda.
Our results provides some
justification for a mechanism in use in school playgrounds
around the world.

\section{Notation}
\label{eq:notation}

%We adopt the notation from \cite{bouveretgeneral} with some minor changes.
We have $p$ items and $n$ agents. Each agents has a total
preference order over the items. A \emph{profile} is
an $n$-tuple of such orders.
Agents share the same utility function.
An item ranked in $k$th position has a utility
$g(k)$. For Borda utilities, $g(k)=p-k+1$.
The utility of a set of items is merely the sum
of their individual utilities.
%Utility is additive so the utility of the set of items
%$A$, $g(A)=\sum\limits_{a\in A}g(a)$ for $A\subseteq[p]$.
Preference orders are independent and drawn
uniformly at random from the set of all $p!$ possibilities
(full independence). 
Agents take turns to pick items according to a \emph{policy},
a sequence $\pi=\pi_1\ldots\pi_p\in\{1,2,\ldots,n\}^p$. At the $k$-th step,
agent $\pi_k$ chooses one item from the remaining set.
Without loss of generality, we suppose $\pi_1=1$.
For profile $R$, $u_R(i,\pi)$ denotes the utility
gained by agent $i$ supposing every agent always chooses
the highest ranked item in their ranking
from the available items. We write
$\overline{u_i}(\pi)$ for the expectation of $u_R(i,\pi)$
over all possible profiles.
%:
%\[\overline u_i(\pi)=\sum_{R}\prob{R}\cdot u_R(i,\pi)=\frac{1}{(p!)^2}\sum_{R} u_R(i,\pi).\]
We take an utilitarian standpoint, measuring social welfare by
the sum of the utilities: % of the agents:
$\sw_R(\pi)=\sum_{i=1}^nu_R(i,\pi)$. By linearity of expectation,
the expected utilitarian social welfare is $\overline{\sw}(\pi)=\sum_{i=1}^n\overline{u_i}(\pi)$. To help compute the expected utilities,
we need a sequence $\gamma_k$ given by $\gamma_1=\gamma_2=1$,
$\gamma_k=\prod\limits_{j=1}^{\lfloor(k-1)/2\rfloor}\frac{2j+1}{2j}$
for $k\geqslant 3$ and $\overline\gamma_k=\gamma_k/k$ for all $k$.
%By Stirling's formula,
Asymptotically
$\gamma_k=\sqrt{\frac{2k}{\pi}}+O\left(\frac{1}{\sqrt k}\right).$
To simplify notation,
we suppose empty sums are zero and empty products
are one.
\section{Computing the Expected Social Welfare}\label{sec:expectations}
Bouveret and Lang~\citeyear{bouveretgeneral}
conjectured that it is NP-hard to compute
the expected social welfare of a given policy.
This calculation
takes into account a super-exponential number of
possible profiles. Nevertheless, as we show here,
the expected utility of each agent can be computed in
just $O(np^2)$ time for an arbitrary utility
function, and $O(np)$ time for Borda utilities.
We begin with this last case, and then extend
the results to the general case.

Let $\mathcal P^n_p$ denote the set of all policies
of length $p$ for $n$ agents. For $p\geqslant 2$, we define an operator $\mathcal P^n_p\to\mathcal P^n_{p-1}$ mapping $\pi\mapsto\tilde\pi$, by deleting the the first entry. More precisely,
%\[\tilde\pi_i=\pi_{i+1}\qquad\text{for }i\in[p-1].\]
$\tilde\pi_i=\pi_{i+1}$ for $ i\in \{1,\dots,p-1\}$.
For example, $\pi=1211$ and $\tilde\pi = 211$.
%\begin{example}
%Consider a policy $\pi = 1211$, $\pi\in\mathcal P_4$. Then $\tilde\pi = 211$. \qed
%\end{example}
%
%Then we have the following recursions:
\begin{lemma}\label{lem:recursion}
For Borda scoring, $n\geqslant 2$ agents, $p\geqslant 2$ items and $\pi\in\mathcal P^n_p$ with $\pi_1=1$,
we have:
% with $\pi_1=1$, the expected utilities for all agents are
\begin{align*}
  \overline u_{1}(\pi)&=p+\overline u_{1}(\tilde\pi), & \overline u_i(\pi)&=\frac{p+1}{p}\overline u_i(\tilde\pi), i\neq 1
\end{align*}
and these values can be computed in $O(np)$ time.
\end{lemma}
\begin{proof}
Agent 1 picks her first item, giving her a utility of $p$. After that, from her perspective, it's the standard game on $p-1$ items with policy $\tilde\pi$, so she expects to get an utility of $\overline{u_1}(\tilde\pi)$. This proves the first equation. For the other agents, it is more involved. Let $i\in\{2,\ldots,n\}$ be a fixed agent. For $q\in \{1,\dots,p\}$, let $a_i(q,\pi)$ denote the probability that under policy $\pi$ agent $i$ gets the item with utility $q$. Note that this probability does not depend on the utility function but only on the ranking: it is the probability that agent $i$ gets the item of rank $p-q+1$ in her preference order. By the definition of expectation,
\begin{equation}\label{eq:expectation}
\overline{u_i}(\pi)=\sum_{q=1}^pa_i(q,\pi)q.
\end{equation}
There are three possible outcomes of the first move of agent 1 with respect to the item that has utility $q$ for agent $i$. With probability $(q-1)/p$, agent 1 has picked an item with utility less than $q$ (for agent $i$), with probability $(p-q)/p$, agent 1 has picked an item with utility more than $q$, and with probability $1/p$ it was the item of utility equal to $q$. In the first case there are only $q-2$ items of utility less than $q$ left, hence the probability for agent $i$ to get the item of utility $q$ is $a_i(q-1,\tilde\pi)$. In the second case there are still $q-1$ items of value less than $q$, hence the probability to get the item of utility $q$ is $a(q,\tilde\pi)$. In the third case, the probability to get the item of utility is zero, and together we obtain
\begin{equation}\label{eq:prob_rec}
a(q,\pi)=\frac{q-1}{p}a_i(q-1,\tilde\pi)+\frac{p-q}{p}a_i(q,\tilde\pi).
\end{equation}
Substituting this into~(\ref{eq:expectation}) yields
\begin{multline*}
\overline{u_i}(\pi) = \sum_{q=1}^p\left[\frac{q-1}{p}a_i(q-1,\tilde\pi)+\frac{p-q}{p}a_i(q,\tilde\pi)\right]q \\
=\sum_{q=1}^p\frac{(q-1)q}{p}a_i(q-1,\tilde\pi)+ \sum_{q=1}^p \frac{(p-q)q}{p}a_i(q,\tilde\pi)
\end{multline*}
In the first sum we substitute $q'=q-1$ and this yields
\begin{align*}
\overline{u_i}(\pi)
&= \sum_{q'=0}^{p-1}\frac{q'}{p}\cdot a_i(q',\tilde\pi)\cdot(q'+1)\ +\ \sum_{q=1}^p\frac{p-q}{p}\cdot a_i(q,\tilde\pi)\cdot q
\end{align*}
The first term in the first sum  and the last term in the second sum are equal to zero, so they can be omitted and we obtain
\begin{align*}
\overline{u_2}(\pi)
&= \sum_{q'=1}^{p-1}\frac{q'}{p}\cdot a_i(q',\tilde\pi)\cdot(q'+1)\ +\ \sum_{q=1}^{p-1}\frac{p-q}{p}\cdot a_i(q,\tilde\pi)\cdot q\\
&= \sum_{q=1}^{p-1}a_i(q,\tilde\pi)\left[\frac{q}{p}\cdot(q+1)+\frac{p-q}{p}\cdot q\right] \\
&= \frac{p+1}{p}\sum_{q=1}^{p-1}a_i(q,\tilde\pi)\cdot q = \frac{p+1}{p}\overline{u_2}(\tilde\pi).
\end{align*}
The time complexity follows immediately from the recursions.
%from the recursive definitions of $\overline{u_1}(\pi)$, $\overline{u_2}(\pi)$ and $\overline{sw}(\pi)$.
\end{proof}
\begin{example}\label{exm:compute_utils_sw_recur}
Consider two agents with Borda utilities and the policies $\pi^1 = 121212$ and $\pi^2= 111222$. We compute expected utilities and expected social welfare for each of them using Lemma~\ref{lem:recursion}. Table~\ref{table:exm:rec_policies} shows
results up to two decimal places.
\begin{table}\label{table:exm:rec_policies}
\scriptsize{
  \centering
  \begin{tabular}{|c|@{}c|@{}c|@{}c|@{}c|@{}c|@{}c|@{}c|@{}c|}
       \hline
        & \multicolumn{4}{|c|}{$ \pi^1 = 121212$} &  \multicolumn{4}{|c|}{$ \pi^2 = 111222$} \\
      \hline
      & $\pi$ & $\overline{u_1}(\pi)$ & $\overline{u_2}(\pi)$ & $\overline{\sw}(\pi)$& $\pi$ & $\overline{u_1}(\pi)$ & $\overline{u_2}(\pi)$ & $\overline{\sw}(\pi)$\\
    \hline
      \hline
    1 & 2 & 0 & 1 & 1&  2 & 0 & 1 & 1\\
    2 & 12 & 2 & 1.5  & 3.5 & 22 & 0 & 3  & 3 \\
    3 & 212 & 2.67 & 4.5 & 7.17&  222 &0  & 6 & 6\\
    4 & 1212 & 6.67 & 5.63 & 12.3&  1222 & 4 & 7.5 & 11.5\\
    5 & 21212 & 8 & 10.63& 18.63&  11222 & 9 & 9& 18\\
    6 & 121212 & 14 & 12.4 & 26.4 & 111222 & 15 & 10.5 & 25.5\\
    \hline
  \end{tabular}
  \caption{ Expected utilities and expected utilitarian social welfare computation for
  $\pi^1 = 121212$ and $\pi^2= 111222$}\label{exp:exp_utils_sw}
  }
\end{table}
%Note that expected values computed in all examples in the paper coincide with the results obtained by the brute-force search algorithm from~\cite{bouveretgeneral_onlinecomput}. \qed
Note that expected values computed in all examples in the paper coincide with the results obtained by the brute-force search algorithm from~\cite{bouveretgeneral}.
% available at \url{http://recherche.noiraudes.net/en/sequences.php}.\qed
\end{example}
Due to the linearity of Borda scoring the probabilities $a_i(q,\pi)$ in the proof of Lemma~\ref{lem:recursion} cancel, and this will allow us to solve recursions explicitly and to prove our main result about the optimal policy for Borda scoring in Section~\ref{sec:optimality}.

In the general case, we can still compute the expected utilities $\overline{u}(i,\pi)$, %$i=1,\ldots,n$, 
and thus $\overline{sw}(\pi)$, but we need the probabilities $a_i(q,\pi)$ from the proof of Lemma~\ref{lem:recursion}: $a_i(q,\pi)$ is the probability that under policy $\pi$, agent $i$ gets the item ranked at position $p-q+1$ in her preference order. Computing these probabilities using~(\ref{eq:prob_rec}) 
adds a factor of $p$ to the runtime.
\begin{lemma}\label{lem:recursion_gen}
For $n\geqslant 2$ agents, $p\geqslant 2$ items, a policy $\pi\in\mathcal P^n_p$ and an arbitrary scoring function $g$, the expected utility for agent $i$ is
\begin{align*}
\overline u_i(\pi)=\sum_{q=1}^pa_i(q,\pi)g(q)% & &\overline {sw}(\pi) = \sum_{i=1}^n \overline u_i(\pi)
\end{align*}
and can be computed in $O(np^2)$ time.
\end{lemma}

Lemma~\ref{lem:recursion_gen} allows us to resolve an open question from~\cite{bouveretgeneral}.
\begin{corollary}
For $n$ agents and an arbitrary scoring utility function $g$,
the expected utility of each agent, as well as the expected
utilitarian social welfare can be computed in polynomial time.
\end{corollary}

For some special policies, the recursions in Lemma~\ref{lem:recursion} can be solved explicitly. A particularly interesting policy is the strictly alternating one (denoted $\Alt$) $\pi=123\ldots n123\ldots n123\ldots n\ldots.$
\begin{proposition}\label{prop:expectation_alternating}
Let $\pi$ be the strictly alternating policy of length $p$ starting with $1$. The expected utilities and utilitarian social welfare for two agents and Borda scoring are
\begin{align*}
  \overline{\sw}(\pi) &= \frac13\big[(2p-1)(p+1) + \gamma_{p+1}\big] \\%= \\
%  & \frac{2p^2+p-1}{3}+\frac13\sqrt{\frac{2p}{\pi}}+O\left(\frac{1}{\sqrt p}\right) \\
 (\overline{u_1}(\pi), \overline{u_2}(\pi))  &=
\begin{cases}
\left(\frac{p(p+1)}{3},\frac{p^2-1}{3} + \frac13\gamma_{p+1}\right) & \text{if $p$ is even,} \\
\left(\frac{p(p+1)}{3} + \frac13\gamma_{p+1},\frac{p^2-1}{3}\right) & \text{if $p$ is odd.}
\end{cases}
\end{align*}
For $n$ agents %($n$ fixed) 
the expectations are, for all $i\in\{1,\ldots,n\}$,
\begin{align*}
\overline u_i(\pi)&=\frac{p^2}{n+1}+O(p),  &  \overline{\sw}(\pi) &= \frac{np^2}{n+1}+O(p).
\end{align*}
\end{proposition}
\begin{proof}
A proof for the two agents case can be done by  straightforward induction using the recursions from Lemma~\ref{lem:recursion}. For $n$ agents we outline the proof. The full proof is given in the Appendix \ref{sec:proof_lemma_recursion_borda_gen} of the online version~\cite{tronline}.
First, we solve the recursions for the expected utility of agent $i$
when the number of items is $p\equiv i-1\pmod n$. Using these values, we approximate the expected utility for the remaining combinations of $p$ and $i$.
\end{proof}
\begin{example}
Consider the policy $\pi^1$ from Example~\ref{exm:compute_utils_sw_recur}.
Using Proposition~\ref{prop:expectation_alternating},
 we get: $\overline{u_1}(\pi) = \frac{6(6+1)}{3} = 14$, $\overline{u_2}(\pi) = \frac{6^2-1}{3} + \frac13\gamma_{6+1} =11.67 + 2.19/3 = 12.4$ and $\overline{\sw}(\pi) = 26.4$. These values coincide with results in Table~\ref{table:exm:rec_policies}. 
\end{example}
For a fixed number $n$ of agents, we write the number of items as $p=kn+r$ with $0\leqslant r<n$. We call a policy $\pi=\pi_1\ldots\pi_p$ \emph{balanced} if $\{\pi_{in+1},\ldots,\pi_{in+n}\}=\{1,2,\ldots,n\}$ for all $i\in\{0,\ldots,k-1\}$ and $\left\lvert\{\pi_{kn+1},\ldots,\pi_{kn+r}\}\right\rvert=r$. For any balanced policy, the expected utility of any agent lies between that of agents 1 and $n$ in the alternating policy. Thus every agent has expected utility $p^2/(n+1)+O(p)$.

%%%%%%%%%%%%%%%%%%%
% alternating and reversing policy
%%%%%%%%%%%%%%
%For the alternating and reversing policy (denoted $\AltRev$) $\pi=122112211221\ldots$ the expressions are simple  only for even numbers of items. Exact values of expected utilities are given in the Appendix as they are involved and approximate values are sufficient for our analysis.
% \xxx[nina]{TODO: should we explain why we considered this policy?}
%\begin{proposition}\label{prop:expectaition_alternating_reverse}
%  Let $\pi$ be the alternating and reversing policy starting with $1$. The expected social welfare is
%\[\overline{\sw}(\pi)=\frac{2p^2+p}{3}+O(\sqrt p).\]
%The more precise asymptotics that correspond to the individual utilities as given in Proposition \ref{prop:expectaition_alternating_rev} in the Appendix.
%%and the expected individual utilities are
%%\begin{align*}
%%      (\overline{u_1}(\pi),  \overline{u_2}(\pi))  &= \left(\frac{p(2p+1)}{6}+O(\sqrt p),  \frac{p(2p+1)}{6}+O(\sqrt p)\right).
%%\end{align*}
%\end{proposition}
%Both Propositions \ref{prop:expectation_alternating} and \ref{prop:expectaition_alternating_reverse} are proved by the straightforward inductions using the recursions from Lemma~\ref{lem:recursion}.

\section{Comparison with Other Mechanisms}\label{sec:comparison}
We compare with two other allocation mechanisms that
provide insight into the efficiency of the alternating
policy. The best preference mechanism ($\BestPref$)
allocates each item to the agent who prefers it most
breaking ties at random. The random mechanism ($\Random$)
allocates items by flipping a coin for each individual item.
\begin{proposition}\label{prop:expectaition_bestpref}
For two agents, and Borda utilities, the expected utilitarian social
welfare of  $\BestPref$ is $\frac{(p+1)(4p-1)}{6}.$
For $n$ agents, it is $\frac{np^2}{n+1}+\frac{p}{2}+O(1).$
\end{proposition}
\begin{proof}
Due to space limitations we only present a proof for two agents. The proof for $n$ agents is again given in
the online Appendix \ref{sec:proof_prop_expectaition_bestpref}.
%~\cite{tronline}.
Let $S_p$ be the set of all permutations of $\{1,2,\ldots,p\}$. Then the expected utilitarian social welfare is
\[\frac{1}{p!}\sum_{a\in S_p}\sum_{i=1}^p\max\{i,a_i\} = \frac{1}{p!}\sum_{i=1}^p\sum_{a\in S_p}\max\{i,a_i\}.\]
We split the inner sum into two sums: one for all permutations $a=(a_1,\ldots,a_p)$ with $a_i\leqslant i$, and one for the remaining permutations. Hence, we get
\[\frac{1}{p!}\sum_{i=1}^p\left[\sum_{a\in S_p\,:\,a_i\leqslant i}i+\sum_{a\in S_p\,:\,a_i> i}a_i\right].\]
We compute the value of each inner sum separately. In the first sum each term equals $i$, so we have to determine the number of terms. For $a_i$ there are $i$ possible values $1,2,\ldots,i$, and for a fixed value of $a_i$ there are $(p-1)!$ permutations of the remaining values. So the first sum has $(p-1)!i$ terms of value $i$, hence it equals $(p-1)!i^2$. The second sum contains for each $j\in\{i+1,\ldots,p\}$ exactly $(p-1)!$ terms of value $a_i=j$, hence it equals $(p-1)!\sum_{j=i+1}^pj=(p-1)!(p(p+1)/2-i(i+1)/2)$.
%To compute the value of the first sum we observe that for all permutations of values $[p] \setminus i$ we have to consider all possible positions of $a_i$. Hence, we have $(p-1)!$ permutations of values $[p] \setminus i$  and for each of these permutations we have $i$ possible positions for values $a_i$ so that $a_i \leqslant i$. The utility gained in each of these positions is $i$.
%To compute the value of the second sum we again consider all permutations of values $[p] \setminus i$, which is $(p-1)!$, and all possible positions of $a_i$ for each permutation, which are $j = i+1,\ldots,n$. The utility gained in each of these positions is exactly $ j$. Hence, $\sum_{a\in S_p\,:\,a_i> i}a_i = (p-1)!\sum_{j=i+1}^pj$. This gives
So the expected utilitarian social welfare is
\begin{multline*}
\frac{1}{p!}\sum_{i=1}^p (p-1)!\left[i^2+\frac12p(p+1)-\frac12i(i+1)\right] \\
=\frac{1}{2p}\sum_{i=1}^p\left[p(p+1)+i^2-i\right]=\frac{(p+1)(4p-1)}{6}.\qedhere
\end{multline*}
\end{proof}
\begin{proposition}\label{prop:expectaition_random}
For $n$ agents, and Borda utilities, the expected utilitarian social
welfare of  $\Random$ is $\frac{p(p+1)}{2}.$
\end{proposition}
\begin{proof}
As the probability of each agent obtaining the $i$th item is $1/n$, the expected utilitarian social welfare is $n\sum_{i=1}^p\frac{i}{n}= \frac{p(p+1)}{2}.$
\end{proof}

%Table~\ref{table:sw_diff_mechanisms} shows expected social welfare for these mechanisms
%Proofs are given in the Appendix.
%\begin{table}
%  \centering
%  \scriptsize{
%  \begin{tabular}{|c|c|c|c|}
%    \hline
%    % after \\: \hline or \cline{col1-col2} \cline{col3-col4} ...
%    \multicolumn{4}{|c|}{ $ \overline{\sw}(\pi)$ } \\
%    \hline
%   \Alt & \AltRev& \BestPref & \Random \\
%   \hline
% $\frac{2p^2+p}{3}+O(\sqrt p)$ & $\frac{2p^2+p}{3}+O(\sqrt p)$&  {$\frac{4p^2+3p}{6} + O(1)$} & {$\frac{p^2+p }{2}$} \\
%% $\frac13\sqrt{\frac{2p}{\pi}}+O\left(\frac{1}{\sqrt p}\right)$ &$\frac13\sqrt{\frac{p}{\pi}}+O\left(\frac{1}{\sqrt p}\right)$&  & $$ %\\
%    \hline
%  \end{tabular}
%  }
%  \caption{Expected social welfare for different allocation mechanisms.}\label{table:sw_diff_mechanisms}
%\end{table}

Table~\ref{table:sw_diff_mechanisms} summarizes the expected utilitarian social welfares for these mechanisms.
\begin{table}
  \centering
  \scriptsize{
  \begin{tabular}{|c|c|c|}
    \hline
    \multicolumn{3}{|c|}{ $ \overline{\sw}(\pi)$ } \\
    \hline
   \Alt &\BestPref & \Random \\
   \hline
 $\frac{np^2}{n+1}+O(p)$ & $\frac{np^2}{n+1}+\frac{p}{2}+O(1)$&  {$\frac{p^2+p }{2}$}\\
    \hline
  \end{tabular}
  }
  \caption{Expected utilitarian social welfare for different mechanisms.}\label{table:sw_diff_mechanisms}
\end{table}
Clearly, $\BestPref$ is an upper bound on the expected utilitarian social welfare for any allocation mechanism. As in~\cite{bouveretgeneral},
we define asymptotic optimality of a sequence of policies $(\pi^{(p)})_{p=1,2,\ldots}$ where $\pi^{(p)}$ is a policy for $p$ items by
\[\lim\limits_{p\to\infty}\frac{\overline{\sw}(\pi^{(p)})}{\max_{\pi \in \mathcal P_p}\overline{\sw}(\pi)}=1.\]
As can be seen from the table,  $\Alt$ is an asymptotically optimal policy. By the observation in the end of the previous section the same is true for any balanced policy, and this implies Proposition 5 in~\cite{bouveretgeneral}. However, the proof in \cite{bouveretgeneral} is
incorrect as it implies that the expected utility is $p^2/n+O(1)$ for every agent which contradicts our upper bound for \BestPref. See Appendix \ref{sec:bouveret_lang_gap} %~\cite{tronline} 
for a detailed discussion of the gaps in the proof.
%Their proof claims that the expected social welfare for
%a balanced policy is $p^2+O(n)$. In view of the upper bound from \BestPref\ this cannot be true, and we were not able to find an argument for the asymptotic optimality of balanced policies along the lines indicated in~\cite{bouveretgeneral}.
Of course, for any given $p$ and preference
orderings, $\Alt$ may not give the maximal utilitarian social welfare possible.
%A similar example can be constructed for  $\AltRev$.
\begin{example}
Consider two agents and six items with the following preferences: $1 > 2 > 3 > 4 > 5 > 6$ and $1 > 6 > 2 > 3 > 4 > 5$. The $\Alt$ policy gives items $\{1,2,4\}$  and $\{6,3,5\}$ to agents 1 and 2 respectively. Hence, the total welfare is $(6+5+3) + (5+3+1) = 23$. Consider a policy $\pi = 121111$ which gives the following items to agents: $\{1,2,3,4,5\}$  and $\{6\}$. The total welfare is now $(6+5+4+3+2) + (6) = 25$. 
\end{example}
The $\Random$ mechanism gives the worst expected utilitarian social
welfare among the three mechanisms. Moreover, as $n$ increases the expected utilitarian social welfare produced by  $\Random$ declines compared with the other two mechanisms: $\lim\limits_{p\to\infty}\frac{\overline{\sw}(\Random^{(p)})}{\overline{\sw}(\BestPref^{(p)})}=\frac{n+1}{2n}$.

With two agents, the expected loss using
$\Alt$ compared to $\BestPref$ (which requires full
revelation of the preference orders)
is less than $p/6$.  In particular, with high probability
$\Alt$ yields an utilitarian social welfare very close to the upper bound.  %\xxx[nina]{would it be possible to state a similar result for $n$ agents?}
\begin{proposition}
For two agents and any $\varepsilon>0$, with probability at least $1-\varepsilon$,
$\Alt$ is a $(1-\frac{1}{3p\varepsilon})$-approximation of the optimal expected utilitarian social welfare.
\end{proposition}
\begin{proof}
Let the random variables $z_1$ and $z_2$ denote the utilitarian social welfare for $\Alt$ and $\BestPref$. Then $z_2-z_1\geqslant 0$ and for the expectations we have $\textbf{E}(z_1) = \frac{2p^2+p-1+\gamma_p}{3}>\frac{2p^2+p}{3}$ and $ \textbf{E}(z_2) = \frac{4p^2+3p-1}{6}$.

% \begin{align*}
%    \textbf{E}(z_1) &= \frac{2p^2+p-1+\gamma_p}{3}>\frac{2p^2+p}{3}, & \textbf{E}(z_2) &= \frac{4p^2+3p-1}{6}.
%  \end{align*}
So $\textbf{E}(z_2-z_1)<p/6$, and by Markov's inequality
\[\textbf{P}\left(z_2-z_1\geqslant\frac{p}{6\varepsilon}\right)\leqslant\frac{p/6}{p/(6\varepsilon)}=\varepsilon.\]
Writing it multiplicatively, with probability at least $1-\varepsilon$,
\[\frac{z_1}{z_2}>\frac{p^2/2-p/(6\varepsilon)}{p^2/2}=1-\frac{1}{3p\varepsilon}.\qedhere\]
\end{proof}

%\xxx[thomas]{It might be better to state the multiplicative version instead: With high probability $z_1/z_2$ is close to 1.}
A similar result holds for more than two agents.
\begin{proposition}\label{prop:error_estimate_many_agents}
For $n$ agents, there exists a constant $C$ such that for every $\varepsilon>0$ with probability at least $1-\varepsilon$ $\Alt$ is a $(1-\frac{C}{p\varepsilon})$-approximation of the optimal expected utilitarian social welfare.
\end{proposition}
%\xxx[thomas]{It should be possible to get some estimates of $C$. Is it worth doing this?}

%\xxx[nina]{TODO: should we have more discussion on comparison of these  mechanisms?}

%Finally, we will consider two special extreme profiles where preferences are complectly correlated, when
%agents have identical preferences, and completely anti-correlation, when the preference order of
%the second agent is the reverse of the preference order of the first agent.
%The first profile yields the worst allocation with a social welfare of $p(p+1)/2$.
%The second profile gives the best allocation when the expected social welfare of $p(3p+2)/4$ for even $p$ and $(p(3p+2)-1)/4$ for odd $p$, and so does the mechanism that allocates every item to the agent that likes it most.
% \xxx[nina]{TODO: should we explain why we consider these profiles?}

\section{Optimality of the Alternating Policy}\label{sec:optimality}

\begin{figure*}[!htbp]  \centering
    \includegraphics[width=1\textwidth]{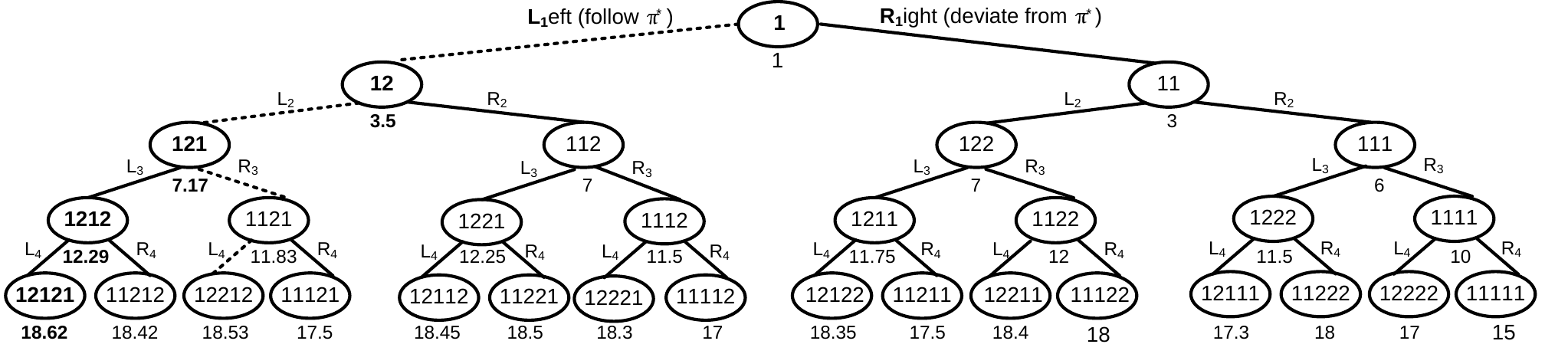}\\
    \caption{The policy tree of depth 5. \label{fig:policies_tree}}
\end{figure*}
We now consider the problem of finding the policy
that maximizes the expected utilitarian social welfare for Borda utilities.
Bouveret and Lang~\citeyear{bouveretgeneral}
stated that this is an open question,
and conjectured that this problem is NP-hard.
%However, they did show that $\Alt$ is asymptotically optimal as $p\to\infty$.
We close this problem, by proving that $\Alt$ is
in fact the optimal policy for {\em any} given $p$ with two agents.
%This resolve the conjecture by Bouveret and Lang. First we state the result.
%The main result of this section is the following theorem.
\begin{theorem}\label{thm:optimality}
The expected utilitarian social welfare is maximized by the alternating policy for two agents supposing Borda utilities and
the full independence assumption.
\end{theorem}
Note that by linearity of expectation this implies optimality of the alternating policy for every linear scoring function $g(k)=\alpha k+\beta$ with $\alpha,\beta\in\reals$, $\alpha\leqslant 0$. In particular,
the result also holds for quasi-indifferent scoring
where $g(k)=N+(p-k+1)$ for large $N$.

In the following let $\pi^*_p$ always be the alternating policy of length $p$. We also recall that due to symmetry we can only consider policies that starts with $1$, e.g. policy $212$ is equivalent to $121$.
%~\cite{bouveretgeneral}.
%
To prove Theorem~\ref{thm:optimality} we need to prove that for any policy $\pi$ of length $p$ the expected utilitarian social welfare is smaller or equal to the expected utilitarian social welfare of $\pi^*_p$. That is, $\dsw_\pi = \overline{\sw}(\pi)-\overline{\sw}(\pi^*_p) \leqslant 0$. We proceed in two steps. First, we describe $\dsw_\pi$ recursively, by representing the policy $\pi$ in terms of its deviations from $\pi^*_p$. Second, given the recursive description of $\dsw_\pi$, we prove by induction that this difference is never positive (Proposition~\ref{prop:reformulation}). The proof is not trivial as the natural inductive approach to derive $\dsw_\pi\leqslant 0$ from $\dsw_{\tilde\pi}\leqslant 0$ does not go through. Hence, we will prove a stronger result in Theorem~\ref{thm:real_induction} that implies Proposition~\ref{prop:reformulation} and Theorem~\ref{thm:optimality}.

\paragraph{Recursive definition.}
To obtain a recursive definition of $\dsw_\pi$, we observe that any policy $\pi$ can be written in terms of its deviations from $\Alt$ policy $\pi^*$. We explain this idea using the following example. Consider a policy $\pi = 1121$. There are two ways to extend $\pi$ with a prefix
%of length one
to obtain policies of length 5:  $\pi' = 11121$ and $\pi'' = 21121$ which is equivalent to $\pi'' = 12212$. We say that $\pi'' = 12212$ follows $\Alt$ in extending $\pi$ as its prefix is $(12)$ which coincides with the alternation step. We say that $\pi' = 11121$ deviates from $\Alt$ in extending $\pi$ as its prefix is $(11)$ which does not correspond to the alternation step.

Next we define a notion of \emph{the policy tree}, which is  a balanced binary tree, that represents all possible policies in terms of deviations from $\pi^*$. The main purpose of this notion is to explain intuitions behind our derivations and  proofs. We start with the policy $(1)$, which is the root of the tree. We expand a policy to the left by a prefix of length one. We can \emph{follow} the strictly alternation policy by expanding (1) with  prefix $2$. This gives policy $(21)$ which is equivalent to $(12)$ due to symmetry. Alternatively, we can \emph{deviate} from $\Alt$ by expanding (1) with  prefix $1$. This gives policy $(11)$. This way we obtain all policies of length $2$. We can continue expanding the tree from $(12)$ and $(11)$ following the same procedure and keeping in mind that we break symmetries by remembering only polices that start with $1$. The following example show all polices of length at most $5$. By convention, given a policy $\pi$ in a node of the tree we say that we follow $\Alt$ on the left branch and deviate from $\Alt$ on the right branch.
\begin{example}
Figure~\ref{fig:policies_tree} shows a tree which represents all policies of length at most $5$. A number below each policy shows the value of the expected utilitarian social welfare for this policy. As can be seen from the tree, $\Alt$ is the optimum policy for all $p$. Consider, for
example, $\pi = 12212$. We can obtain this by deviations from $\pi^*_5$ (shown as the dashed path): $(1) \rightarrow_{L_1} (12) \rightarrow_{L_2} (121) \rightarrow_{R_3} (1121) \rightarrow_{L_4} (12212)$. 
\end{example}

Next we give a formal recursive definition of $\dsw_\pi$.  We recall that from Lemma~\ref{lem:recursion} the recursions for $\Alt$
\begin{align*}
  (\overline{u_1}(\pi^*_p), \overline{u_2}(\pi^*_p)) &=\left (p+\overline{u_2}(\pi^*_{p-1}),  \frac{p+1}{p}\overline{u_1}(\pi^*_{p-1})\right)
\end{align*}
For any $\pi\in\mathcal P_p$, $p\geqslant 2$ we obtain a similar recursion that depends on whether $\pi$ follows or deviates from $\pi^*$ in extension of $\tilde\pi$ at each step. In the first case, the prefix of $\pi$ is $(12)$ and in the second case the prefix is $(11)$. So we have
\begin{align*}
  (\overline{u_1}(\pi), \overline{u_2}(\pi)) &=  \begin{cases} \left(p+\overline{u_2}(\tilde\pi), \frac{p+1}{p}\overline{u_1}(\tilde\pi)\right)  \text{if $\pi = 12 \ldots$,} \\
  \left(p+\overline{u_1}(\tilde\pi),\frac{p+1}{p}\overline{u_2}(\tilde\pi) \right)  \text{if $\pi = 11 \ldots$.} \end{cases}
\end{align*}

Then $(\overline{u_1}(\pi)-\overline{u_1}(\pi^*_p), \overline{u_2}(\pi)-\overline{u_2}(\pi^*_p)) =$
\begin{align*}
     \begin{cases} \left(\overline{u_2}(\tilde\pi)-\overline{u_2}(\pi^*_{p-1}), \frac{p+1}{p}\left(\overline{u_1}(\tilde\pi)-\overline{u_1}(\pi^*_{p-1})\right)\right)  \text{if $\pi = 12 \ldots$,} \\
  \left(\overline{u_1}(\tilde\pi)-\overline{u_2}(\pi^*_{p-1}), \frac{p+1}{p}\left(\overline{u_2}(\tilde\pi)-\overline{u_1}(\pi^*_{p-1}) \right)  \right)  \text{if $\pi = 11 \ldots$.} \end{cases}
\end{align*}

We introduce notations to simplify the explanation. Using Proposition~\ref{prop:expectation_alternating} we define $\delta_p=\overline{u_1}(\pi^*_p)-\overline{u_2}(\pi^*_p)=\frac13\left[p+1+(-1)^{p+1}\gamma_{p+1}\right]$.
We define the sets
\[A_p=\left\{\left(\overline{u_1}(\pi)-\overline{u_1}(\pi^*_p),\overline{u_2}(\pi)-\overline{u_2}(\pi^*_p)\right)\ :\ \pi\in\mathcal P_p\right\}.\]
Note that for an element $(a,b)\in A_p$ corresponding to a policy $\pi\in\mathcal P_p$ we have $a+b = \overline{\sw}(\pi)-\overline{\sw}(\pi^*_p)=\dsw_\pi$. Hence, $\pi$ has a higher expected utilitarian social welfare than $\pi^*_p$ if and only if $a+b>0$.

The recursions above provide a description of the sets $A_p$. We have $A_1=\{(0,0)\}$ because $\pi^*_1$ is the only policy of length 1, and for $p\geqslant 2$ the set $A_p$ consists of the elements $\left(b,\frac{p+1}{p}a\right)$ and $\left(a+\delta_{p-1},\frac{p+1}{p}(b-\delta_{p-1})\right)$ where $(a,b)$ runs over $A_{p-1}$. Theorem \ref{thm:optimality} is equivalent to the following statement.
\begin{proposition}\label{prop:reformulation}
Let $A_1=\{(0,0)\}$ and
\[A_k=\left\{\left(b,\frac{k+1}{k}a\right) \right\}\cup\left\{\left(a+\delta_k,\frac{k+1}{k}(b-\delta_k)\right) \right\}\]
for $k\geqslant 2$ where $(a,b)\in A_{k-1}$, $\displaystyle\delta_k=\frac{1}{3}\left(k+(-1)^k\gamma_k\right)$. Then $a+b\leqslant 0$ for all $(a,b)\in\bigcup_k A_k$.
\end{proposition}
Figure~\ref{fig:A_k_in_policies_tree} shows the sets $A_k$, $k=1,\ldots,4$ in the policy tree.
\begin{figure}[!htbp]  \centering
    \includegraphics[width=0.48\textwidth]{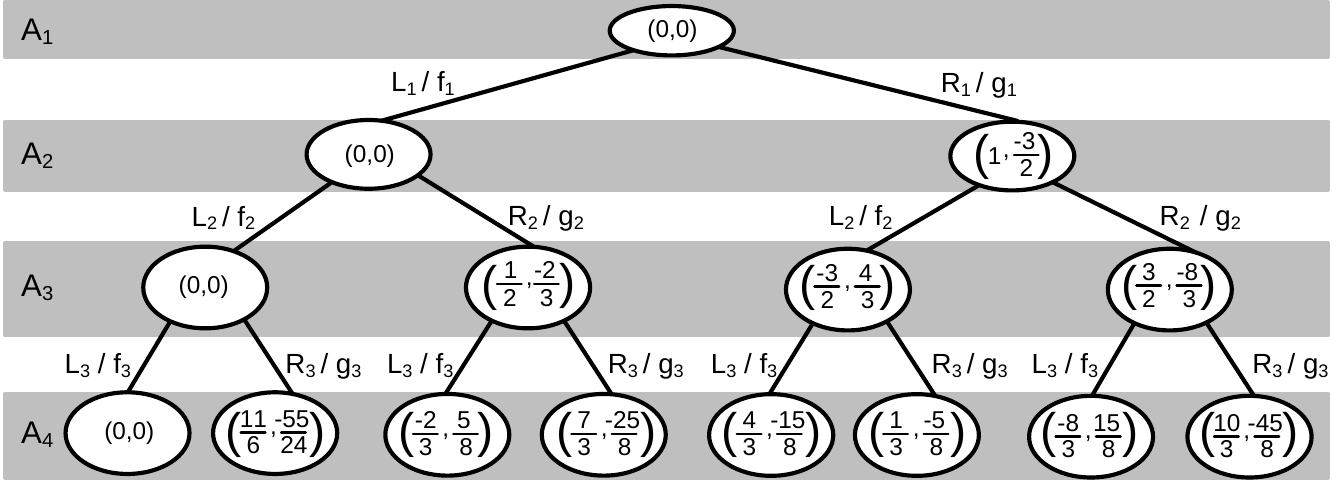}\\
    \caption{The sets $A_k$, $k=1,\ldots,4$ in the policy tree. \label{fig:A_k_in_policies_tree}}
\end{figure}

\paragraph{Proving optimality.}
%Unfortunately, it is not possible to prove Proposition~\ref{prop:reformulation} by induction as the inductive hypothesis is not strong enough.
We might try to prove Proposition~\ref{prop:reformulation} inductively by deriving $a+b\leqslant 0$ for the point $(a,b)\in A_p$ corresponding to policy $\pi$ from $a'+b'\leqslant 0$ for $(a',b')\in A_{p-1}$ corresponding to policy $\tilde\pi$. Unfortunately, the induction hypothesis is too weak as the following example shows.
% The following example demonstrates its weakness.
\begin{example}\label{exm:weak_induction_hyph}
Assume $(a',b')=(-12, 11.9) \in A_{10}$ corresponding to some policy $\pi\in\mathcal P_{10}$. Let $(a,b)\in A_{11}$ be obtained from $(-12, 11.9)$ by deviating from $\pi^*_{11}$. With $\delta_{11} = 2.7643$ we obtain $a+b= -9.2357+9.9662 > 0$. Thus $(a',b')$ satisfies Proposition \ref{prop:reformulation} while $(a,b)$ violates it. 
\end{example}
%The main problem with this example is that it is not possible to obtain $(-12, 11.9) \in A_{10}$. However, Proposition~\ref{prop:reformulation} statement does not make any claims about $A_k$ except of $a_k + b_k \leqslant 0$.

To remedy this problem we would like to strengthen the proposition, for example by proving $a+b\leqslant f(a,b)$ for all $(a,b)\in\bigcup_k A_k$ where $f$ is some function with $f(x,y)\leqslant 0$ for all $(x,y)$. The difficulty of finding such a function is indicated by Figure~\ref{fig:wild_set} showing the set $A_{10}$. Different markers distinguish the points arising from $A_9$ by following $\pi^*$ from those deviating from $\pi^*$.
\begin{figure}[htb]
  \centering
  \begin{minipage}{.49\linewidth}
    \includegraphics[width=\textwidth]{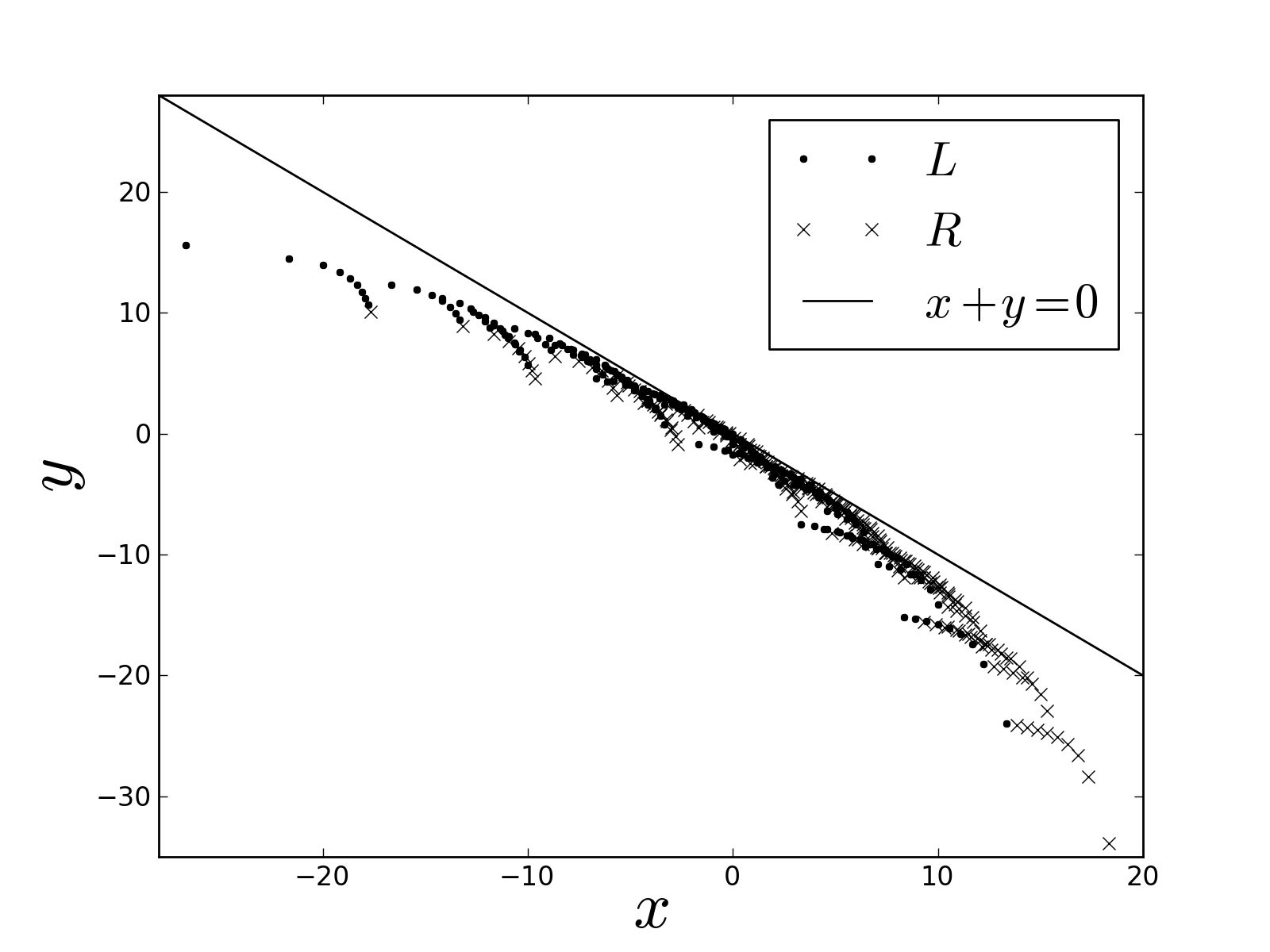}
  \end{minipage}\hfill
  \begin{minipage}{.49\linewidth}
    \includegraphics[width=\textwidth]{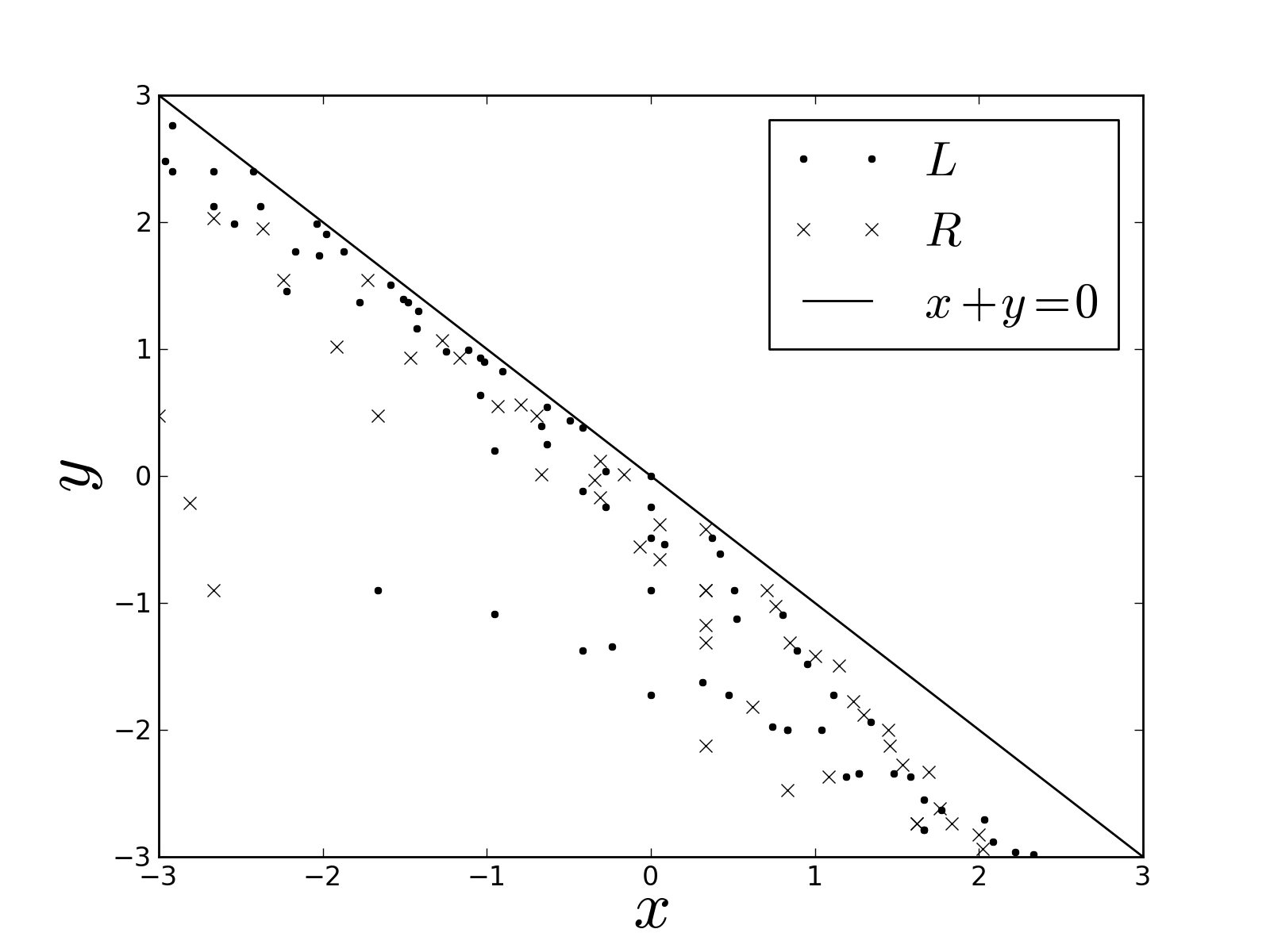}
  \end{minipage}
  \caption{The set $A_{10}$ and a more detailed view of the region around the origin.}
  \label{fig:wild_set}
\end{figure}

%One natural way would be to enforce  stronger bounds on the sum, e.g. $a_k+b_k < -c_k$, where $c_k$ is some sequence of positive integers that we can compute independently of $\pi$. However, we were not able to find such a sequence because we could not identify any additional to the statement of Proposition~\ref{prop:reformulation} behavioral property of $a_k$, $b_k$ and their sum that we can exploit. To demonstrate this irregularity, we plotted sequences $a_k$, $b_k$ and $a_k+b_k$, $k=1,\ldots,66$ for some policy $\pi$ of length 66 in Figure~\ref{fig:noproperties}. As can be seen from the plot there is no evident regular property that any of these sequences obey.
% \begin{figure}[htb]  \centering
%     \includegraphics[width=0.5\textwidth]{noproperties.pdf}\\
%     \caption{Behavior of $a_k$, $b_k$ and $a_k+b_k$, $k=1,\ldots,66$ for some policy $\pi$. \label{fig:noproperties}}
% \end{figure}

%The key idea of our proof is to build an induction argument over sets of policies rather than a single policy as in Proposition~\ref{prop:reformulation}. This allows us to use these sets to derive stronger bounds on expected social welfare ($a_k + b_k$) for some policies in these sets. We describe the sets first and show schematically how we used them in the outline of the proof.

The key idea of our proof is to strengthen Proposition \ref{prop:reformulation} in another direction. We describe this strengthening first and then outline the induction argument. The technical details of the proof are presented in the online Appendix \ref{sec:proof_main_thorem}.
%~\cite{tronline}. 
Consider a policy $\pi$ that is represented by a node $n_{\pi}$ at level $k$ in the policy tree. Instead of requiring the inequality $a+b\leqslant 0$ only for the point $(a,b)\in A_k$ that corresponds to policy $\pi$, we also require it for (i) all policies that lay on the path that follow only the right branches from $n_{\pi}$ and (ii) all polices that lay on the path that starts from $n_{\pi}$ by following the left branch once and then only follow the right branches. To formalize this idea, for $k\geqslant 1$ we define functions $f_k,g_k:\reals^2\to\reals^2$ by $f_k(x,y) = \left(y,\frac{k+2}{k+1}x\right)$ and  $g_k(x,y) = \left(x+\delta_{k+1},\frac{k+2}{k+1}(y-\delta_{k+1})\right)$. Note that $A_{k+1}=f_k(A_{k})\cup g_k(A_{k})$ for all $k\geqslant 1$, as $f_k$ encodes the case when we follow the left branch and $g_k$ -- the right branch.  Figure~\ref{fig:A_k_in_policies_tree} illustrates this correspondence.
We also consider iterated compositions of these functions. For every $k\geqslant 1$ let $G_{k0}$ denote the identity on $\reals^2$, i.e. $G_{k0}(x,y)=(x,y)$, and for $m\geqslant 1$ let $G_{km}$ denote the function
\[G_{km}=g_{k+m-1}\circ g_{k+m-2}\circ\cdots\circ g_{k}.\]

Applying $G_{km}$ to the point $(a,b)\in A_k$ corresponding to $\pi\in\mathcal P_k$ gives the point $(a',b')\in A_{k+m}$ which corresponds to the policy $\pi'\in\mathcal P_{k+m}$ that is obtained from $\pi$ by following the right branch $m$ times.
%This composition corresponds to the situation when from a node at the $k$th  level of the tree we follow the right branch $m$ times.
For all $k\geqslant 1$ and $m\geqslant 1$, we define the function $F_{km}=G_{k+1,m-1}\circ f_k$. $F_{km}$ corresponds to starting in level $k$, following the first left branch and then $m-1$ right branches. For $(x,y)\in A_k$, $G_{km}(x,y)\in A_{k+m}$ for $m\geqslant 0$ and $F_{km}(x,y)\in A_{k+m}$ for $m\geqslant 1$. Proposition \ref{prop:reformulation} is a consequence of the following theorem.
\begin{theorem}\label{thm:real_induction}
Let $A_1=\{(0,0)\}$ and
\begin{align*}
A_{k+1}=f_k(A_k)\cup g_k(A_k)= & \left\{\left(y,\frac{k+2}{k+1}x\right)\ \right\}\cup \\ \left\{\left(x+\delta_{k+1},\frac{k+2}{k+1}(y-\delta_{k+1})\right)\right\}
\end{align*}
for $k\geqslant 1$ where  $(x,y)\in A_{k}$,  $\displaystyle\delta_k=\frac13\left(k+(-1)^k\gamma_k\right)$. Then for every $k\geqslant 1$ and every $(x,y)\in A_k$ the following statements are true.
\begin{enumerate}
\item For all $m\geqslant 0$, if $(x',y')=G_{km}(x,y)$ then $x'+y'\leqslant 0$.
\item For all $m\geqslant 1$, if $(x',y')=F_{km}(x,y)$ then $x'+y'\leqslant 0$.
\end{enumerate}
%In particular, the first statement with $m=0$ implies Proposition \ref{prop:reformulation} and hence Theorem \ref{thm:optimality}.
\end{theorem}
\begin{proof}[Proof sketch]
\begin{figure}[htb]  \centering
    \includegraphics[width=0.4\textwidth]{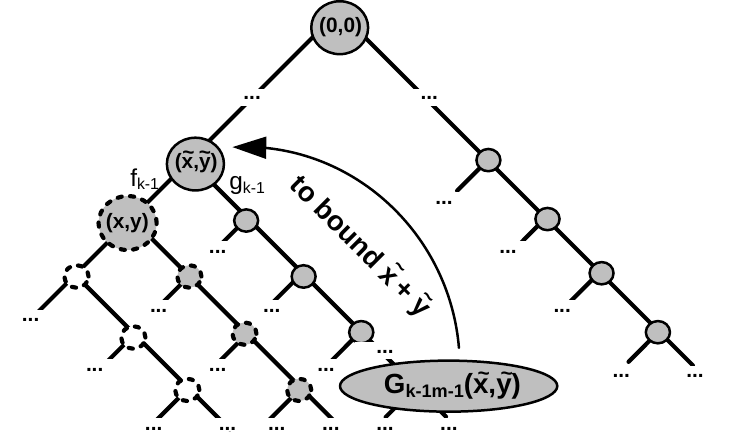}\\
    \caption{Schematic representation of the proof of Theorem~\ref{thm:real_induction}. \label{fig:tree_scheme_proof}}
\end{figure}
We provide the full proof in the Appendix~\ref{subsec:actual_proof}.
%~\cite{tronline}. 
We give a high-level overview. We start with a few technical lemmas to derive an explicit description of functions $(x',y') = F_{km} (x,y)$ and $(x'',y'') = G_{km} (x,y)$. This gives us explicit expressions for the sums $x'+y'$ and $x''+y''$ in terms of $x$ and $y$. Then we proceed through the induction proof. We summarize the induction step here. Suppose the statements of the theorem are already proved for all sets $A_l$ with $l<k$. Let $(x,y)$ be an arbitrary element of $A_k$. Suppose $(x,y)=f_{k-1}(\tilde x,\tilde y)$ for some $(\tilde x,\tilde y)\in A_{k-1}$ (the case $(x,y)=g_{k-1}(\tilde x,\tilde y)$ is similar). Figure~\ref{fig:tree_scheme_proof} shows $(\tilde x, \tilde y)$ and $(x,y)$ that is obtained from $(\tilde x, \tilde y)$  by following the left branch. By the induction hypothesis, $x'+y'\leqslant 0$ whenever $(x',y')\in\{G_{k-1,m}(\tilde x, \tilde y),F_{k-1,m}(\tilde x ,\tilde y)\}$. The corresponding nodes are highlighted in gray in Figure~\ref{fig:tree_scheme_proof}. To complete the induction step we need to show $x'+y'\leqslant 0$ for $(x',y')\in\{G_{km} (x,y),F_{km} (x,y)\}$. The corresponding nodes are indicated by dashed circles. The result for $(x',y')=G_{km}(x,y)$ (gray and dashed) follows immediately as $(x',y')=G_{km}(x,y)=F_{k-1,m+1}(\tilde x, \tilde y)$.
%These polices are already  highlighted in gray in Figure~\ref{fig:tree_scheme_proof}.
For $(x',y')=F_{km}(x,y)$ we first express $x'+y'$ in terms of $\tilde x$ and $\tilde y$. Then, by induction $x''+y''\leqslant 0$ for $(x'',y'')=G_{k-1,m-1}(\tilde x,\tilde y)$. Inverting the representation of $x''+y''$ in terms of $\tilde x$ and $\tilde y$ we derive a bound $\tilde x + \tilde y \leq -c(m)\leqslant 0$, depending on $m$, and this stronger bound is used to prove $x'+y'\leqslant 0$ for $(x',y')=F_{km}(x,y)$.
%\xxx[nina]{TODO: nicer description of the proof.}
\end{proof}

The extension of Theorem~\ref{thm:real_induction} to $n$ agents is not straightforward. Firstly, it requires deriving exact recursions for the expected utility for an arbitrary $p$. This is not trivial, as Proposition~\ref{prop:expectation_alternating} only provides asymptotics. %If this problem is resolved, we expect that an extension of Theorem~\ref{thm:real_induction} will be technically very tedious as the current proof for two agents is involved.
An easier extension might be to other utility
functions. The alternating policy is not optimal for
all scoring functions. For example, it is not optimal for
the $k$-approval scoring function
which has $g(i)=1$ for $i\leqslant k$ and 0 otherwise.
However, we conjecture that $\Alt$ is optimal for all
convex scoring functions (which includes
lexicographical scoring). %, i.e. scoring functions with $g(k+1)-g(k)\geqslant g(k)-g(k-1)$ for all $k\in\{2,3,\ldots,p-1\}$. In particular, this includes lexicographic scoring. It seems plausible that this can be proved along the lines of our proof of Theorem \ref{thm:optimality}, but so far we did not succeed in carrying out the induction step.

%\xxx[nina]{TODO: more discussions?}

\section{Strategic Behaviour}\label{sec:strategic}

So far, we have supposed agents sincerely pick
the most valuable item left. However, agents can
sometimes improve their utility by picking
less valuable items. To understand such strategic behaviour, we 
view this as a finite repeated game with perfect information.
\cite{kohler1971class} proves that we can compute the
subgame perfect Nash equilibrium for the alternating
policy with two agents by simply reversing
the policy and the preferences and playing the
game backwards. More recently,
\cite{comsoc2012} prove this holds for any
policy with two agents. %, not just the alternating policy.
%We define $rev(\pi)$ as the reversal of the policy $\pi$.
%

We will exploit such reversal symmetry.
We say that a policy $\pi$ is {\em reversal symmetric} if
and only
the reversal of $\pi$, after interchanging the agents if necessary,
equals $\pi$. The policies $1212$ and $1221$ are reversal symmetric,
but %the policy 
$1121$ is not.
%We define the reversal of the two agent
%profile $(o_1,o_2)$ as $(rev(o_2),rev(o_1))$.
%We let $u_R^{t}(i,\pi)$ and $u_R^s(i,\pi)$ be
%the  utility for agent $i$ assuming truthful behaviour and
%strategic behaviour, respectively.
%We have $u_R^s(i,\pi)=u_{rev(R)}^t(i,rev(\pi))$.
The next result follows quickly by expanding
and rearranging expressions for the expected
utilitarian social welfare using the fact that we
can compute strategic play by simply reversing the
policy and profile and supposing truthful
behaviour.

\begin{theorem}\label{thm:strategic_behaviour}
For two agents and any utility function,
any reversal symmetric policy that
maximizes the expected utilitarian social
welfare for truthful behaviour also maximizes
the expected utilitarian social welfare for strategic behaviour.
\end{theorem}

As the alternating policy is reversal
symmetric, it follows that the alternating policy is also optimal
for strategic behaviour.
Unfortunately, the generalisation of these results
to more than two agents is complex.
Indeed, for an unbounded number of agents,
computing the subgame perfect Nash equilibrium becomes
PSPACE-hard \cite{knwxaaai13}.

\section{Conclusions}

Supposing additive utilities, and full independence
between agents,
we have shown that we can compute the expected utility
of a sequential allocation procedure in polynomial time
for any utility function. Using this result,
we have proven that the expected utilitarian social welfare for Borda
utilities
is maximized by the alternating policy in which two agents
pick items in a fixed order.
We have argued that this mechanism remains optimal when
agents behave strategically.
There remain open several important questions.
For example, is the alternating policy optimal for more
than two agents? What happens with non-additive
utilities?
\eject
\bibliographystyle{named}
\bibliography{social_choice}
~\\~\\~\\~\\~\\~\\~\\~\\~\\~\\~\\~\\~\\~\\~\\~\\~\\~\\~\\~\\~\\~\\~\\~\\~\\~\\~\\~\\~\\~\\~\\~\\~\\~\\~\\~\\~\\~\\~\\~\\~\\~\\~\\~\\~\\~\\~\\~\\~\\~\\~\\
\eject

\section*{Acknowledgments}

The authors are supported by the Australian Governments Department of Broadband, Communications and the Digital Economy, the Australian Research Council and the Asian Office of Aerospace Research and Development through grant AOARD-124056.

\section*{Appendix}

\begin{appendix}

%\subsection{Alternating and reversing policy}\label{apx:altrevpolicy}
%We give a full statement of Proposition~\ref{prop:expectaition_alternating_reverse}
%that identifies exact individual utilities for  alternating and reversing policy.
%\begin{proposition}\label{prop:expectaition_alternating_rev}
%  Let $\pi$ be the alternating and reversing policy starting with $1$. Both agents have Borda utility functions.
%  The expected social welfare is
%\[\overline{\sw}(\pi)=\frac{4p^2+2p-1}{6}+\frac13\sqrt{\frac{p}{\pi}}+O\left(\frac{1}{\sqrt p}\right)\]
%and the expected individual utilities are
%  \begin{enumerate}
%  \item If $p\equiv 0\pmod 4$ then
%    \begin{align*}
%      \overline{u_1}(\pi) &= \frac{p(2p+1)}{6}, \\ \overline{u_2}(\pi) &= \frac{(p+1)(2p-1)}{6}+\frac{p+1}{3p}\gamma_{p/2-1}.
%    \end{align*}
%  \item If $p\equiv 2\pmod 4$ then
%    \begin{align*}
%      \overline{u_1}(\pi) &= \frac{p(2p+1)}{6}+\frac13\gamma_{p/2}, \\ \overline{u_2}(\pi) &= \frac{(p+1)(2p-1)}{6}.
%    \end{align*}
%  \item If $p\equiv 1\pmod 4$, say $p=4k+1$, then
%    \begin{align*}
%      \overline{u_1}(\pi) &= p+\sum_{j=0}^{k-1}(8j+3)\prod_{i=0}^{k-j-1}\frac{p-4i}{p-4i-2}, \\ \overline{u_2}(\pi) &= \frac{p+1}{p}\sum_{j=0}^{k-1}(8j+7)\prod_{i=0}^{k-j-2}\frac{p-4i-2}{p-4i-4}.
%    \end{align*}
%  \item If $p\equiv 3\pmod 4$, say $p=4k+3$, then
%    \begin{align*}
%      \overline{u_1}(\pi) &= p+\sum_{j=0}^{k-1}(8j+7)\prod_{i=0}^{k-j-1}\frac{p-4i}{p-4i-2}, \\ \overline{u_2}(\pi) &= \frac{p+1}{p}\sum_{j=0}^{k}(8j+7)\prod_{i=0}^{k-j-1}\frac{p-4i-2}{p-4i-4}.
%    \end{align*}
%  \end{enumerate}
%\end{proposition}

\section{Proof of Proposition \ref{prop:expectation_alternating} for $n$ agents}\label{sec:proof_lemma_recursion_borda_gen}
We determine the asymptotic behaviour of the expected utilities for the strictly alternating policy of length $p$
\[\pi=123\ldots n123\ldots n123\ldots n\ldots.\]

%\xxx[nina]{TODO: Thomas, could you please change notations in this section to be consistent with the rest of the paper? I am not sure what the best way to do this.}

As the policy is now determined by the number of items we simplify notation by letting $u_{ip}$ be the expected utility of agent $i$ for the allocation of $p$ items. Then $u_{11}=1$, $u_{21}=u_{31}=\cdots=u_{n1}=0$ and
\begin{align*}
u_{1p} &= p+u_{n,p-1},& u_{ip}&=\frac{p+1}{p}u_{i-1,p-1}\quad(i=2,\ldots,n)
\end{align*}
for $p\geqslant 2$. Decoupling these recursions we get for the first agent
\begin{align*}
u_{1p} &= p,\ p\leqslant n, & u_{1p}&=p+\frac{p}{p-n+1}u_{1,p-n}, \ p>n.
\end{align*}
and this allows us to write down the expected utility exactly for one residue class modulo $n$ per agent.
\begin{proposition}\label{prop:exact_values}
For $i\in\{1,2,\ldots,n\}$, if $p\equiv i-1\pmod n$ then the expected utility of agent $i$ equals
\[u_{ip}=\frac{(p-i+1)(p+1)}{n+1}\]
\end{proposition}
\begin{proof}
  We start with $i=1$ and prove by induction that $u_{1p}=\frac{p(p+1)}{n+1}$ for all $p\equiv 0\pmod n$. The induction starts at $p=n$ with $u_{1p}=n=p(p+1)/(n+1)$. For $p>n$ with $p\equiv 0\pmod n$, we have by induction
  \begin{align*}
u_{1p}&=p+\frac{p}{p-n+1}u_{1,p-n}\\ &=p+\frac{p}{p-n+1}\frac{(p-n)(p-n+1)}{n+1} =\frac{p(p+1)}{n+1}.
\end{align*}
Now we proceed by induction on $i$. For $i\geqslant 2$ our recursion gives
  \begin{align*}
u_{ip}= &\frac{p+1}{p}u_{i-1,p-1} \\ =&\frac{p+1}{p}\frac{((p-1)-(i-1)+1)((p-1)+1)}{n+1}\\ =&\frac{(p-i+1)(p+1)}{n+1}.\qedhere
\end{align*}
\end{proof}
For the remaining residue classes mod $n$ we provide asymptotic statements.
\begin{proposition}\label{prop:asymptotics}
For fixed $n$ the expected utility of agent $i\in\{1,2,\ldots,n\}$ equals
\[u_{ip}=\frac{p^2}{n+1}+O(p).\]
\end{proposition}
\begin{proof}
For $p\equiv i-1\pmod n$ this follows from Proposition \ref{prop:exact_values}. Otherwise let $k$ be the unique element of $\{1,\ldots,n-1\}$ such that $p-k\equiv i-1\pmod n$. The following estimates prove the claim. From
\begin{align*}
  u_{i,p-k} &\leqslant u_{ip}\leqslant u_{i,p+(n-k)}\\
\end{align*}
it follows that
\begin{align*}
u_{ip} \geqslant & \frac{(p-k-i+1)(p-k+1)}{(n+1)},  \\
u_{ip} \leqslant & \frac{(p+(n-k)-i+1)(p+(n-k)+1)}{n+1}.
\end{align*}
This gives
\begin{align*}
\frac{p^2}{n+1}-2p&\leqslant u_{ip}\leqslant \frac{p^2}{n+1}+2p+n.\qedhere
\end{align*}
\end{proof}
\begin{corollary}
The expected utilitarian social welfare for the alternating policy is $\displaystyle\frac{np^2}{n+1}+O(p)$.
\end{corollary}

\section{Proof of Proposition \ref{prop:expectaition_bestpref}  for $n$ agents }\label{sec:proof_prop_expectaition_bestpref}
\begin{proposition}\label{prop:asymp_upper_bound}
Let $\pi$ be the $\BestPref$ policy with $n$ agents using
Borda utility functions.
The expected utilitarian social welfare is
\[\frac{np^2}{n+1}+\frac{p}{2}+O(1).\]
\end{proposition}
\begin{proof}
The expected utilitarian social welfare for this procedure is the expected value of the random variable
\[X=\sum_{q=1}^p\max_{1\leqslant i\leqslant n}\alpha_{iq}\]
where the $n$ vectors $\vect{\alpha_i}=(\alpha_{i1},\alpha_{i2},\ldots,\alpha_{ip})$ are random permutations of the set $\{1,2,\ldots,p\}$ that are drawn independent and uniform from the set of all $p!$ permutations. The interpretation is that we fix an order of the items, and $a_{iq}$ is the value of item $q$ for agent $i$ according to her random preference order. We decompose $X$ as a sum of random variables %\footnote{I guess the expectation of $X$ is somewhere in the literature, but finding it takes probably longer than just doing the calculations.}
\[X_q=\max_{1\leqslant i\leqslant n}\alpha_{iq}\]
and calculate the expected value of these. For the probability that $X_q$ takes value $j$ we can write
\begin{align*}
\prob{X_q=j}=&\sum_{k=1}^n\binom{n}{k}\left(\frac1p\right)^k\left(\frac{j-1}{p}\right)^{n-k} \\ =&\left(\frac1p+\frac{j-1}{p}\right)^n-\left(\frac{j-1}{p}\right)^n\\
=&\left(\frac{j}{p}\right)^n-\left(\frac{j-1}{p}\right)^n.
\end{align*}
The expected value of $X_q$ is
\begin{multline*}
\expect{X_q}=\sum_{j=1}^pj\left(\left(\frac{j}{p}\right)^n-\left(\frac{j-1}{p}\right)^n\right)\\
=\frac{1}{p^n}\sum_{j=1}^pj\left[j^n-(j-1)^n\right] = \frac{1}{p^n}\left[\sum_{j=1}^pj^{n+1}-\sum_{j=0}^{p-1}(j+1)j^n\right] \\
=\frac{1}{p^n}\left[\sum_{j=1}^pj^{n+1}-\sum_{j=0}^{p-1}j^{n+1}-\sum_{j=0}^{p-1}j^{n}\right] \\
= \frac{1}{p^n}\left[p^{n+1}-\frac{1}{n+1}(p-1)^{n+1}-\frac12(p-1)^n+O(p^{n-1})\right] \\
= \frac{1}{p^n}\left[\frac{n}{n+1}p^{n+1}+\frac12p^n+O(p^{n-1})\right]=\frac{np}{n+1}+\frac12+O(1/p).
\end{multline*}
Now summation over $q$ yields the result.
%\xxx[nina]{TODO: should we have a high level explanation for the derivations above?}
%\xxx[thomas]{I don't think that's worth the effort. We are certainly not the first who have computed this expectation.}
\end{proof}

\section{Analysis of the proof of the asymptotic optimality of balanced policies}\label{sec:bouveret_lang_gap}
Bouveret and Lang claim that every sequence of balanced policies is asymptotically optimal. In fact, their statement is even a bit stronger in that they do not require $\pi_{kn+1},\ldots,\pi_{kn+q}$ to be pairwise distinct (their $\theta$ is any agent sequence). Our argument in Section \ref{sec:comparison} can be adapted to yield this stronger result. In the present section we point out some serious gaps in the proof given by Bouveret and Lang.

For $i=1,\ldots,k$ let the sequence of stages $(i-1)n+1,(i-1)n+2,\ldots,in$ be called the $i$-th \emph{round}, i.e. for a balanced policy, in every round each agent picks exactly one item.

Bouveret and Lang start with the observation that in the first round the first agent gets utility $p$ and the second one $(p^2-1)/p=(1+o(1))p$. They say that the third agent gets $\Theta(p)$ which is too weak for what they want to derive: If the third agent would get $p/2$ this would be $\Theta(p)$ but not $p+O(p^{-1})$ which is what the claim next. The proof of this expectation of $p+O(p^{-1})$ is already nontrivial and could be done as follows. The $j$-th agent of the first round gets her most preferred item with probability $\tbinom{p-1}{j-1}/\tbinom{p}{j-1}=\frac{p-j+1}{p}$ and her second most preferred item with probability $\tbinom{p-2}{j-2}/\tbinom{p}{j-1}=\frac{(j-1)(p-j+1)}{p(p-1)}$, so her expected utility from the first round is at least
\[p-j+1+\frac{(j-1)(p-j+1)}{p}=p-\frac{(j-1)^2}{p}=p+O(p^{-1}).\]

Then Bouveret and Lang continue by stating that in the second round the starting agent gets her second preferred item with probability $1-\frac{n-1}{p-1}$. This is only true if the second round starts with the same agent as the first round. Otherwise one has to take into account the probability, that the first agent of the second round took her second favourite item already in the first round (because her first choice was not available). For instance if the starting agent of the second round was second in the first round then her probability to pick her second preferred item in the second round equals
\[\frac{p-2}{p}\cdot\frac{\binom{p-3}{n-2}}{\binom{p-2}{n-2}}=1-\frac{n}{p}.\]
They argue that the starting agent of round two gets utility at least $(1-\frac{n-1}{p-1})(p-1)=p-1+O(p^{-1})$, and this last equality is clearly wrong. In order to show that every agent from the second round gets utility at least $p-1+O(p^{-1})$ it would be necessary to consider not only the probability that the agent gets her second most preferred item in the second round, but also probabilities for other items (just like for the first round we had to take into account the most preferred and the second most preferred item). It seems possible that this can be done (maybe just the third preferred item is sufficient), but it is in no way obvious how to do it.

It seems to be very difficult to generalize this from the second round to the following rounds. Their claim is that in round $i$ every agent gets utility $p-i+1+O(p^{-1})$. It might be that this is true (although highly non-obvious) for bounded $i$, but Bouveret and Lang use this statement for all $i$ up to $k$ (which tends to infinity with $p$). Even if the $p-i+1+O(p^{-1})$ utility for round $i$ would be correct, their calculation of the total utility $\left[p+O(p^{-1})+\cdots+(p+k-1)+O(p^{-1})\right]$ of any agent is still wrong. It should be:
\begin{multline*}
\sum_{i=1}^{k}(p-i+1+O(p^{-1}))=kp-\frac{k(k-1)}{2}+O(1)\\ =\frac{(2n-1)p^2}{2n^2}+\frac{p}{2n}+O(1).
\end{multline*}

\section{Proof of Theorem \ref{thm:real_induction}}\label{sec:proof_main_thorem}
\subsection{Technical lemmas}\label{subsec:preliminaries}
We start with the observation that
\begin{equation}\label{eq:overlinegamma_recursion}
\overline\gamma_{k+1}=\begin{cases} \overline\gamma_k & \text{if $k$ is even},\\ \frac{k}{k+1}\overline\gamma_k & \text{if $k$ is odd}.\end{cases}
\end{equation}
In the following lemma we describe the functions $G_{km}$ and $F_{km}$ explicitly.
\begin{lemma}\label{lem:aux_1}
For $k\geqslant 1$, $m\geqslant 0$, and $(x',y')=G_{km}(x,y)$, and $(x'',y'')=F_{km}(x,y)$, we have
  \begin{align*}
    x' &= x+\sum_{j=1}^m\delta_{k+j}, & \\
    y' &= (k+m+1)\left(\frac{y}{k+1}-\sum_{j=1}^m\frac{\delta_{k+j}}{k+j}\right).
  \end{align*}
For $k\geqslant 1$, $m\geqslant 1$, and $(x'',y'')=F_{km}(x,y)$, we have
  \begin{align*}
    x'' &= y+\sum_{j=2}^{m}\delta_{k+j}, & \\
    y'' &= (k+m+1)\left(\frac{x}{k+1}-\sum_{j=2}^{m}\frac{\delta_{k+j}}{k+j}\right).
  \end{align*}
\end{lemma}
\begin{proof}
The expression for $x'$ follows immediately from the definitions. For $y'$ we proceed by induction on $m$. The start for $m=0$ is trivial: $y'=y$. So assume $m\geqslant 1$ and let $(\tilde x,\tilde y)=G_{k,m-1}(x,y)$. Then $(x',y')=g_{k+m-1}(\tilde x,\tilde y)$, and using the induction hypothesis we obtain
\begin{multline*}
y'=\frac{k+m+1}{k+m}\left(\tilde y-\delta_{k+m}\right)\\ = \frac{k+m+1}{k+m}\left((k+m)\left(\frac{y}{k+1}-\sum_{j=1}^{m-1}\frac{\delta_{k+j}}{k+j}\right)-\delta_{k+m}\right)\\
= (k+m+1)\left(\frac{y}{k+1}-\sum_{j=1}^m\frac{\delta_{k+j}}{k+j}\right).
\end{multline*}
Finally, the expressions for $x''$ and $y''$ follow immediately from $(x'',y'')=G_{k+1,m-1}\left(y,\frac{k+2}{k+1}x\right)$.
\end{proof}
In the following two lemmas we calculate how the functions $G_{km}$ and $F_{km}$ affect the coordinate sum $(x,y)$.
\begin{lemma}\label{lem:aux_2}
Let $k\geqslant 1$, $m\geqslant 0$, and $(x',y')=G_{km}(x,y)$.
\begin{enumerate}
\item If $k$ is odd, then
  \begin{equation}\label{eq:G_for_odd_k}
  \begin{split}
x'+y' = &x+y+\frac{m}{k+1}y- \frac{m(m+1)}{6}-\\ &  \frac13\sum_{j=1}^{\lfloor(m+1)/2\rfloor}\overline\gamma_{k+2j-1}.
     \end{split}
  \end{equation}
\item If $k$ is even, then
  \begin{equation}\label{eq:G_for_even_k}
    \begin{split}
x'+y'=&x+y+\frac{m}{k+1}y-\frac{m(m+1)}{6}+ \\ &\frac{m}{3}\overline\gamma_{k+1}-\frac13\sum_{j=1}^{\lfloor m/2\rfloor}\overline\gamma_{k+2j}.
      \end{split}
  \end{equation}
\end{enumerate}
\end{lemma}
\begin{proof}
With Lemma \ref{lem:aux_2} we obtain
\begin{equation*}
 \begin{array}{@{}r@{\quad}cl@{}}
 x'+y' &= &x+\sum_{j=1}^m\delta_{k+j}+  \\
 &&(k+m+1)\left(\frac{y}{k+1}-\sum_{j=1}^m\frac{\delta_{k+j}}{k+j}\right)  \\
 &= &x+y+\frac{m}{k+1}y+ \\
 &&\sum_{j=1}^m\left(1-\frac{k+m-1}{k+j}\right)\delta_{k+j} \\
 &=&x+y+\frac{m}{k+1}y- \\
&&\frac13\sum_{j=1}^m\frac{m+1-j}{k+j}\left(k+j+(-1)^{k+j}\gamma_{k+j}\right)  \\
&=&x+y+\frac{m}{k+1}y-\frac13\sum_{j=1}^m(m+1-j) - \\
&&\frac13\sum_{j=1}^m(-1)^{k+j}(m+1-j)\overline\gamma_{k+j} \\
&=&x+y+\frac{m}{k+1}y-\frac{m(m+1)}{6} - \\
&&\frac13\sum_{j=1}^m(-1)^{k+j}(m+1-j)\overline\gamma_{k+j}.
 \end{array}%
\end{equation*}

Using (\ref{eq:overlinegamma_recursion}) it is easy to check that for odd $k$
\[\sum_{j=1}^m(-1)^{k+j}(m+1-j)\overline\gamma_{k+j}=\sum_{j=1}^{\lfloor(m+1)/2\rfloor}\overline\gamma_{k+2j-1},\]
and for even $k$,
\[\sum_{j=1}^m(-1)^{k+j}(m+1-j)\overline\gamma_{k+j}=-m\overline\gamma_k+\sum_{j=1}^{\lfloor m/2\rfloor}\overline\gamma_{k+2j},\]
and this concludes the proof.
\end{proof}
\begin{lemma}\label{lem:aux_3}
Let $k\geqslant 1$, $m\geqslant 1$, and $(x',y')=F_{km}(x,y)$.
\begin{enumerate}
\item If $k$ is odd, then
  \begin{equation}\label{eq:F_for_odd_k}
   \begin{split}
    x'+y'=x+y+\frac{m}{k+1}x-\frac{(m-1)m}{6}+\\
    \frac{m-1}{3}\overline\gamma_{k+2} -\frac13\sum_{j=1}^{\lfloor(m-1)/2\rfloor}\overline\gamma_{k+2j+1}.
     \end{split}
  \end{equation}
\item If $k$ is even, then
  \begin{equation}\label{eq:F_for_even_k}
   \begin{split}
    x'+y'=x+y+\frac{m}{k+1}x-\frac{(m-1)m}{6} \\ -\frac13\sum_{j=1}^{\lfloor m/2\rfloor}\overline\gamma_{k+2j}.
  \end{split}
  \end{equation}
\end{enumerate}
\end{lemma}
\begin{proof}
We proceed exactly as in the proof of Lemma \ref{lem:aux_2}.
\end{proof}
In the proof of our main result we need some rough bounds on the numbers $\overline\gamma_k$. The following weak estimates will be sufficient for the induction step in the proof of Theorem \ref{thm:real_induction} below.
\begin{lemma}\label{lem:inequality}
For $k\geqslant 2$, $m\geqslant 0$ we have
\begin{align*}
\overline\gamma_k-\overline\gamma_{k+2\lfloor(m+1)/2\rfloor}&\leqslant\frac{m(m+1)}{2k}, \\
\overline\gamma_{k+1}-\overline\gamma_{k+2\lfloor m/2\rfloor+1}&\leqslant\frac{m^2}{2k}.
\end{align*}
\end{lemma}
\begin{proof}
For $m=0$ both of these inequalities are trivially true. So assume $m\geqslant 1$. For $i=1,\ldots,\lfloor(m+1)/2\rfloor$, using $\overline\gamma_{k+2i-2}\leqslant 1/2$ we have
\[\overline\gamma_{k+2i-2}-\overline\gamma_{k+2i}\leqslant\left(1-\frac{k+2i-2}{k+2i-1}\right)\overline\gamma_{k+2i-2}\leqslant\frac1{2k},\]
and summation over $i$ yields
\[\overline\gamma_k-\overline\gamma_{k+2\lfloor(m+1)/2\rfloor}\leqslant\frac{m+1}{4k}\leqslant\frac{m(m+1)}{2k}.\]
The second inequality is also trivial for $m=1$. For $m\geqslant 2$ and $i=1,\ldots,\lfloor m/2\rfloor$, using $\overline\gamma_{(k+1)+2i-2}\leqslant 1/2$ we have
\begin{equation*}
 \begin{array}{l}
\overline\gamma_{k+1+2i-2}-\overline\gamma_{(k+1)+2i} \\
\leqslant \left(1-\frac{(k+1)+2i-2}{(k+1)+2i-1}\right)\overline\gamma_{(k+1)+2i-2}\leqslant\frac1{2k},
 \end{array}%
\end{equation*}
and summation over $i$ yields
\[\overline\gamma_{k+1}-\overline\gamma_{k+2\lfloor m/2\rfloor+1}\leqslant\frac{m}{4k}\leqslant\frac{m^2}{2k}.\qedhere\]
\end{proof}
\subsection{The induction argument}\label{subsec:actual_proof}
Proposition \ref{prop:reformulation} is a consequence of the following theorem.
\begin{theorem}\label{thm:real_induction_app}
Let $A_1=\{(0,0)\}$ and
\begin{align*}
A_{k+1}
= &f_k(A_k)\cup g_k(A_k) \\
= &\left\{\left(y,\frac{k+2}{k+1}x\right)\ :\ (x,y)\in A_{k}\right\}\cup \\
&\left\{\left(x+\delta_{k+1},\frac{k+2}{k+1}(y-\delta_{k+1})\right)\ :\ (x,y)\in A_{k}\right\}
\end{align*}

for $k\geqslant 1$ where $\displaystyle\delta_k=\frac13\left(k+(-1)^k\gamma_k\right)$. Then for every $k\geqslant 1$ and every $(x,y)\in A_k$ the following statements are true.
\begin{enumerate}
\item For all $m\geqslant 0$, if $(x',y')=G_{km}(x,y)$ then $x'+y'\leqslant 0$.
\item For all $m\geqslant 1$, if $(x',y')=F_{km}(x,y)$ then $x'+y'\leqslant 0$.
\end{enumerate}
In particular, the first statement with $m=0$ implies Proposition \ref{prop:reformulation} and hence Theorem \ref{thm:optimality}.
\end{theorem}
\begin{proof}
We proceed by induction on $k$. For $k=1$ we only have to consider $(x,y)=(0,0)$. For $(x',y')=G_{1m}(0,0)$, it follows from (\ref{eq:G_for_odd_k}) that
\[x'+y'=-\frac{m(m+1)}{6}-\frac13\sum_{j=1}^{\lfloor(m+1)/2\rfloor}\overline\gamma_{k+2j-1}\leqslant 0.\]
For $(x',y')=F_{1m}(0,0)$, it follows from (\ref{eq:F_for_odd_k}) and $\overline\gamma_3=1/2$ that
\begin{align*}
x'+y'=-\frac{(m-1)m}{6}+\frac{m-1}{3}\overline\gamma_{3}-\frac13\sum_{j=1}^{\lfloor(m-1)/2\rfloor}\overline\gamma_{k+2j+1}\\
=-\frac{(m-1)^2}{6}-\frac13\sum_{j=1}^{\lfloor(m+1)/2\rfloor}\overline\gamma_{k+2j-1}\leqslant 0.
\end{align*}

We now assume that $k>1$ and the statements of the theorem are already proved for all sets $A_l$ with $l<k$. Let $(x,y)$ be an arbitrary element of $A_k$. We distinguish two cases.
\begin{description}
\item[Case 1.] $(x,y)=f_{k-1}(\tilde x,\tilde y)$ for some $(\tilde x,\tilde y)\in A_{k-1}$. If $(x',y')=G_{km}(x,y)=F_{k-1,m+1}(\tilde x,\tilde y)$ then $x'+y'\leqslant 0$ follows immediately from the induction hypothesis applied to $(\tilde x,\tilde y)$. So suppose
\[(x',y')=F_{km}(x,y)=F_{km}\left(\tilde y,\frac{k+1}{k}\tilde x\right).\]
We need to consider the two parities of $k$ separately.
\begin{description}
\item[Odd $k$.] From Lemma \ref{lem:aux_3} it follows that
  \begin{equation}\label{eq:prime_1}
  \begin{split}
    x'+y' &\\
    =&\tilde y+\frac{k+1}{k}\tilde x+\frac{m}{k+1}\tilde y-\frac{(m-1)m}{6}+ \\
    &\frac{m-1}{3}\overline\gamma_{k+2}-\frac13\sum_{j=1}^{\lfloor(m-1)/2\rfloor}\overline\gamma_{k+2j+1}.
  \end{split}
  \end{equation}
By induction $\tilde y+\frac{k+1}{k}\tilde x\leqslant 0$, and using $\overline\gamma_{k+1}\leqslant 1/2$ we conclude that $x'+y'\leqslant 0$ is immediate if $\tilde y\leqslant\frac{(m-1)^2(k+1)}{6m}$. Hence we may assume $\tilde y>\frac{(m-1)^2(k+1)}{6m}$. Let $(x'',y'')=G_{k-1,m-1}(\tilde x,\tilde y)$. By Lemma \ref{lem:aux_2},
\begin{align*}
x''+y''=\tilde x+\tilde y+\frac{m-1}{k}\tilde y-\frac{(m-1)m}{6}+\\
\frac{m-1}{3}\overline\gamma_{k}-\frac13\sum_{j=1}^{\lfloor (m-1)/2\rfloor}\overline\gamma_{k+2j-1},
\end{align*}
and by induction $x''+y''\leqslant 0$. Hence
\begin{align*}
\tilde x+\tilde y\leqslant\frac{(m-1)m}{6}+\frac13\sum_{j=1}^{\lfloor (m-1)/2\rfloor}\overline\gamma_{k+2j-1}- \\
\frac{m-1}{k}\tilde y-\frac{m-1}{3}\overline\gamma_{k},
\end{align*}
and substituting into (\ref{eq:prime_1}) yields together with $\overline\gamma_k>\overline\gamma_{k+2}$,
\begin{multline*}
x'+y'<\frac1k\tilde x+\left(\frac{m}{k+1}-\frac{m-1}{k}\right)\tilde y+ \\
\frac13\left(\overline\gamma_{k+1}-\overline\gamma_{k+2\lfloor(m-1)/2\rfloor+1}\right)\\
=\frac{\tilde x+\tilde y}{k}-\frac{m}{k(k+1)}\tilde y+\frac13\left(\overline\gamma_{k+1}-\overline\gamma_{k+2\lfloor(m-1)/2\rfloor+1}\right).
\end{multline*}
With $\tilde x+\tilde y\leqslant 0$ and $\tilde y>\frac{(m-1)^2(k+1)}{6m}$ this implies
\[x'+y'<\frac13\left(\overline\gamma_{k+1}-\overline\gamma_{k+2\lfloor(m-1)/2\rfloor+1}\right)-\frac{(m-1)^2}{6k},\]
and finally, $x'+y'<0$ by Lemma \ref{lem:inequality}.
\item[Even $k$.] From Lemma \ref{lem:aux_3} it follows that
\begin{equation}\label{eq:prime_2}
\begin{split}
x'+y'=\tilde y+\frac{k+1}{k}\tilde x+\frac{m}{k+1}\tilde y-\frac{(m-1)m}{6} \\
-\frac13\sum_{j=1}^{\lfloor m/2\rfloor}\overline\gamma_{k+2j}.
\end{split}
\end{equation}
By induction $\tilde y+\frac{k+1}{k}\tilde x\leqslant 0$, so $x'+y'\leqslant 0$ is immediate if $\tilde y\leqslant\frac{(m-1)(k+1)}{6}$. Hence we may assume $\tilde y>\frac{(m-1)(k+1)}{6}$. Let $(x'',y'')=G_{k-1,m-1}(\tilde x,\tilde y)$. By Lemma \ref{lem:aux_2},
\begin{align*}
x''+y''= \tilde x+\tilde y+\frac{m-1}{k}\tilde y-\frac{(m-1)m}{6}- \\
\frac13\sum_{j=1}^{\lfloor m/2\rfloor}\overline\gamma_{k+2j-2},
\end{align*}

and by induction $x''+y''\leqslant 0$. Hence
\[\tilde x+\tilde y\leqslant\frac{(m-1)m}{6}+\frac13\sum_{j=1}^{\lfloor m/2\rfloor}\overline\gamma_{k+2j-2}-\frac{m-1}{k}\tilde y,\]
and substituting into (\ref{eq:prime_2}) yields
\begin{align*}
x'+y'<\frac1k\tilde x+\left(\frac{m}{k+1}-\frac{m-1}{k}\right)\tilde y+
\frac13\left(\overline\gamma_{k}-\overline\gamma_{k+2\lfloor m/2\rfloor}\right)\\
=\frac{\tilde x+\tilde y}{k}-\frac{m}{k(k+1)}\tilde y+\frac13\left(\overline\gamma_{k}-\overline\gamma_{k+2\lfloor m/2\rfloor}\right).
\end{align*}

With $\tilde x+\tilde y\leqslant 0$ and $\tilde y>\frac{(m-1)(k+1)}{6}$ this implies
\[x'+y'<\frac13\left(\overline\gamma_{k}-\overline\gamma_{k+2\lfloor m/2\rfloor}\right)-\frac{m(m-1)}{6k},\]
and finally, $x'+y'<0$ by Lemma \ref{lem:inequality}.
\end{description}

\item[Case 2.] $(x,y)=g_{k-1}(\tilde x,\tilde y)$ for some $(\tilde x,\tilde y)\in A_{k-1}$. If $(x',y')=G_{km}(x,y)=G_{k-1,m+1}(\tilde x,\tilde y)$ then $x'+y'\leqslant 0$ follows immediately from the induction hypothesis applied to $(\tilde x,\tilde y)$. So suppose
\[(x',y')=F_{km}(x,y)=F_{km}\left(\tilde x+\delta_k,\frac{k+1}{k}(\tilde y-\delta_k)\right).\]
Again we discuss odd and even $k$ separately.
\begin{description}
\item[Odd $k$.] From Lemma \ref{lem:aux_3} it follows that
\begin{equation}\label{eq:prime_3a}
\begin{split}
x'+y' = & \tilde x+\delta_k+\frac{k+1}{k}(\tilde y-\delta_k)+ \\
&\frac{m}{k+1}(\tilde x+\delta_k)-  \frac{(m-1)m}{6}+ \\
&\frac{m-1}{3}\overline\gamma_{k+2}- \frac13\sum_{j=1}^{\lfloor(m-1)/2\rfloor}\overline\gamma_{k+2j+1}.
\end{split}
\end{equation}
By induction $\tilde x+\delta_k+\frac{k+1}{k}(\tilde y-\delta_k)\leqslant 0$, and using $\overline\gamma_{k+2}\leqslant 1/2$ we conclude that $x'+y'\leqslant 0$ is immediate if $m(\tilde x+\delta_k)/(k+1)\leqslant (m-1)^2/6$. So with $\delta_k=\frac13(k-\gamma_k)$ we may assume
\begin{equation}\label{eq:tilde_x_bound}
\tilde x>\frac{(k+1)(m-1)^2}{6m}-\frac{k-\gamma_k}{3}.
\end{equation}
Substituting $\delta_k=\frac13(k-\gamma_k)$ into (\ref{eq:prime_3a}), rearranging terms, and using $\overline\gamma_{k+2}=k\overline\gamma_k/(k+1)$ we obtain
\begin{equation}\label{eq:prime_3}
\begin{split}
x'+y'=& \\
&\tilde x+\tilde y+\frac1k\tilde y+\frac{m}{k+1}\tilde x-\frac{(m-2)(m-1)}{6}- \\
&\frac13\sum_{j=1}^{\lfloor(m-1)/2\rfloor}\overline\gamma_{k+2j+1}-\frac{m}{3(k+1)}+\frac{\overline\gamma_k}{3(k+1)}.
\end{split}
\end{equation}
For $m=1$ we rearrange terms and use $\tilde x+\tilde y\leqslant 0$ and (\ref{eq:tilde_x_bound}) to obtain
\begin{multline*}
x'+y'=\tilde x+\tilde y+\frac1k\tilde y+\frac{1}{k+1}\tilde x-\frac{1}{3(k+1)}+\frac{\overline\gamma_k}{3(k+1)}\\ =\frac{k+1}{k}(\tilde x+\tilde y)-\frac{\tilde x}{k(k+1)}-\frac{1}{3(k+1)}+\frac{\overline\gamma_k}{3(k+1)}\\
\leqslant -\frac{1}{k(k+1)}\left(\tilde x+\frac{k-\overline\gamma_k}{3}\right)<0.
\end{multline*}
For $m\geqslant 2$ let $(x'',y'')=F_{k-1,m-1}(\tilde x,\tilde y)$. By Lemma \ref{lem:aux_3},
\begin{multline*}
x''+y''=\tilde x+\tilde y+\frac{m-1}{k}\tilde x-\frac{(m-2)(m-1)}{6}- \\
\frac13\sum_{j=1}^{\lfloor(m-1)/2\rfloor}\overline\gamma_{k+2j-1},
\end{multline*}
and by induction $x''+y''\leqslant 0$. So
\begin{multline*}
\tilde x+\tilde y\leqslant \frac{(m-2)(m-1)}{6}+\frac13\sum_{j=1}^{\lfloor(m-1)/2\rfloor}\overline\gamma_{k+2j-1}- \\
\frac{m-1}{k}\tilde x,
\end{multline*}
and substituting into (\ref{eq:prime_3}) yields
\begin{multline*}
x'+y'\leqslant\frac{\tilde x+\tilde y}{k}-\frac{m}{k(k+1)}\tilde x-\frac{m}{3(k+1)}+\frac{\overline\gamma_k}{3(k+1)}+ \\
\frac13\left(\overline\gamma_{k+1}-\overline\gamma_{k+2\lfloor(m-1)/2\rfloor+1}\right).
\end{multline*}

With $\tilde x+\tilde y\leqslant 0$ and (\ref{eq:tilde_x_bound}) we obtain
\begin{multline*}
x'+y'\leqslant\frac{m}{3k(k+1)}\left(k-\frac{(k+1)(m-1)^2}{2m}-\gamma_k\right)-\\
\frac{m}{3(k+1)}+\frac{\overline\gamma_k}{3(k+1)}+\frac13\left(\overline\gamma_{k+1}-\overline\gamma_{k+2\lfloor(m-1)/2\rfloor+1}\right)\\
=\frac{(m-1)^2}{6k}+\frac{1-m}{3(k+1)}\overline\gamma_k+\frac13\left(\overline\gamma_{k+1}-\overline\gamma_{k+2\lfloor(m-1)/2\rfloor+1}\right)\\ <\frac13\left(\overline\gamma_{k+1}-\overline\gamma_{k+2\lfloor(m-1)/2\rfloor+1}\right)-\frac{(m-1)^2}{6k},
\end{multline*}
and finally $x'+y'< 0$ by Lemma \ref{lem:inequality}.
\item[Even $k$.] From Lemma \ref{lem:aux_3} it follows that
\begin{equation}\label{eq:prime_4a}
\begin{split}
x'+y' = \tilde x+\delta_k+\frac{k+1}{k}(\tilde y-\delta_k)+ \\
\frac{m}{k+1}(\tilde x+\delta_k)-  \frac{(m-1)m}{6}-\frac13\sum_{j=1}^{\lfloor m/2\rfloor}\overline\gamma_{k+2j}.
\end{split}
\end{equation}
By induction $\tilde x+\delta_k+\frac{k+1}{k}(\tilde y-\delta_k)\leqslant 0$, and we conclude that $x'+y'\leqslant 0$ is immediate if $m(\tilde x+\delta_k)/(k+1)\leqslant (m-1)m/6$. So with $\delta_k=\frac13\left(k+\gamma_k\right)$ we may assume
\begin{equation}\label{eq:tilde_x_bound_2}
\tilde x>\frac{(k+1)(m-1)}{6}-\frac{k+\gamma_k}{3}.
\end{equation}
Substituting $\delta_k=\frac13(k+\gamma_k)$ into (\ref{eq:prime_4a}) and rearranging terms we obtain
\begin{equation}\label{eq:prime_4}
\begin{split}
x'+y'=\tilde x+\tilde y+\frac1k\tilde y+\frac{m}{k+1}\tilde x-\frac{(m-2)(m-1)}{6}+ \\ \frac{m-1}{3}\overline\gamma_k-\frac13\sum_{j=1}^{\lfloor m/2\rfloor}\overline\gamma_{k+2j}-\frac{m}{3(k+1)}(1+\overline\gamma_k).
\end{split}
\end{equation}
For $m=1$ we rearrange terms and use $\tilde x+\tilde y\leqslant 0$ and (\ref{eq:tilde_x_bound_2}) to obtain
\begin{multline*}
x'+y'=\tilde x+\tilde y+\frac1k\tilde y+\frac{1}{k+1}\tilde x-\frac{1+\overline\gamma_k}{3(k+1)}\\
=\frac{k+1}{k}(\tilde x+\tilde y)-\frac{\tilde x}{k(k+1)}-\frac{1+\overline\gamma_k}{3(k+1)}\\
\leqslant -\frac{1}{k(k+1)}\left(\tilde x+\frac{k+\gamma_k}{3}\right)<0.
\end{multline*}
For $m\geqslant 2$ let $(x'',y'')=F_{k-1,m-1}(\tilde x,\tilde y)$. By Lemma \ref{lem:aux_3},
\begin{multline*}
x''+y''=\tilde x+\tilde y+\frac{m-1}{k}\tilde x-\frac{(m-2)(m-1)}{6}+ \\ \frac{m-2}{3}\overline\gamma_{k+1}-\frac13\sum_{j=1}^{\lfloor(m-2)/2\rfloor}\overline\gamma_{k+2j},
\end{multline*}

and by induction $x''+y''\leqslant 0$. So
\begin{multline*}
\tilde x+\tilde y\leqslant \frac{(m-2)(m-1)}{6}+\frac13\sum_{j=1}^{\lfloor(m-2)/2\rfloor}\overline\gamma_{k+2j}- \\
\frac{m-1}{k}\tilde x-\frac{m-2}{3}\overline\gamma_{k+1},
\end{multline*}
and substituting this into (\ref{eq:prime_4}), taking into account $\overline\gamma_{k+1}=\overline\gamma_k$, yields
\begin{multline*}
x'+y'\leqslant \frac{\tilde x+\tilde y}{k}-\frac{m}{k(k+1)}\left(\tilde x+\frac{k+\gamma_k}{3}\right)+ \\ \frac13\left(\overline\gamma_k-\overline\gamma_{k+2\lfloor m/2\rfloor}\right).
\end{multline*}
With $\tilde x+\tilde y\leqslant 0$ and (\ref{eq:tilde_x_bound_2}) we obtain
\[x'+y'<\frac13\left(\overline\gamma_k-\overline\gamma_{k+2\lfloor m/2\rfloor}\right)-\frac{(m-1)m}{6k}\]
and finally $x'+y'< 0$ by Lemma \ref{lem:inequality}. \qedhere
\end{description}
\end{description}
\end{proof}

%\bibliographystyle{named}
%\bibliography{social_choice}

\end{appendix}

%\eject

%\balancecolumns % GM June 2007
% That's all folks!
\end{document}